\documentclass[letterpaper]{article}
\usepackage[preprint]{aaai2027}
\usepackage[hyphens]{url}
\usepackage{graphicx}
\usepackage{amsmath}
\usepackage{amssymb}
\usepackage{amsfonts}
\usepackage{amsthm}
\usepackage{mathtools}
\usepackage{booktabs}
\usepackage{array}
\usepackage{multirow}
\usepackage{pifont}
\usepackage{algorithm}
\usepackage{algpseudocode}
\usepackage{natbib}
\usepackage{caption}
\usepackage{xfp}
\usepackage{placeins}
\newcommand{\projecturl}{%
\leavevmode
\pdfstartlink attr{/Border [0 0 0]} user{/Subtype /Link /A << /S /URI /URI (https://github.com/songdc98/sketchops) >>}%
\textcolor{blue}{\url{https://github.com/songdc98/sketchops}}%
\pdfendlink}
\graphicspath{{./}{figures/}}

\theoremstyle{plain} \newtheorem{theorem}{Theorem}  \newtheorem{proposition}[theorem]{Proposition} \newtheorem{observation}[theorem]{Observation}   \theoremstyle{remark} \newtheorem{remark}[theorem]{Remark}

\newcolumntype{L}[1]{>{\raggedright\arraybackslash}p{#1}}
\newcolumntype{C}[1]{>{\centering\arraybackslash}p{#1}}
\title{Mergeable Model-Side Aggregation States for Long-Context Language Models}
\author{
Dachuan Song\textsuperscript{\rm 1},
Junyu Yin\textsuperscript{\rm 2},
Zechen Hu\textsuperscript{\rm 1},
Xuan Wang\textsuperscript{\rm 1}
}
\affiliations{
\textsuperscript{\rm 1}Department of Electrical and Computer Engineering, George Mason University, Fairfax, VA, USA\\
\textsuperscript{\rm 2}Department of Computer Science, George Mason University, Fairfax, VA, USA
}

\begin{document}

\maketitle

\begin{abstract}
A known limitation of long-context language models is their increasingly unreliable performance in non-additive, set-based aggregation as context length grows. Examples include cardinality estimation, set relationships, and grouped statistics, which widely exist in logs, program outputs, tables, and multi-turn conversations. 
To provide the aggregation state required by these tasks, we introduce a model-side aggregation interface that maintains compact Hash-based HyperLogLog (HLL) sketch states alongside a frozen language model. While the model processes the context, an extractor maps each relevant record to a canonical identity. The identity is then hashed and updates the HLL state.
These states can be merged across context segments and/or read out directly for downstream reasoning, avoiding an additional generate--execute--return cycle.
We validate the proposed approach by setting the HLL state size as 2 KiB (2,048 registers), which does not increase with context length or set cardinality. In a distinct-count experiment involving one million records, the mean relative error was 1.6\%. In a separate merge test, states built from as many as 256 segments produced exactly the same readout as a single pass over the same stream. On 3,969 aggregate-then-reason tasks from 174 source windows, the fixed-budget interface reached 99.2\% accuracy on Gemma 4 (31B, BF16), compared with 100.0\% under exact aggregation; the paired gap was 0.8 percentage points (95\% window-cluster CI: 0.5--1.3 points). On a matched set of 174 items, our method improved over direct full-context reasoning by 63.2 points on Qwen and 56.3 points on Gemma. The corresponding gains over chain-of-thought (CoT) reasoning were 60.9 and 63.2 points, respectively. On a fixed 1,200-task Oolong-Synth subset, our method reached 91.1\% on Qwen and 99.3\% on Gemma. Code is available at \projecturl.

\end{abstract}

\section{Introduction}
\label{sec:intro}

Long-context language models are increasingly used to analyze logs, program outputs,  tables, and multi-turn conversations~\citep{bertsch2025oolong,tooloutputs2026,logcopilot2026,cao2026coding}. 
Although these models can effectively identify relevant records, they tend to become increasingly unreliable at non-additive, set-based aggregation as context length grows~\citep{hsieh2024ruler,bertsch2025oolong}. Examples of such aggregation include cardinality estimation, set relationships, and grouped statistics, all of which may provide essential evidence for downstream reasoning and decision-making.
This limitation arises because the attention and pooling operations used in language models rely extensively on averaging and normalization. While these mechanisms are well suited for summarizing semantic content, they are not necessarily effective for numerical state aggregation, which requires explicitly tracking item uniqueness, duplication, or set union and overlaps. Although learned readouts may approximate these behaviors within short sequence lengths, they may fail when the context length and/or set cardinality extend beyond the training distribution.

Existing approaches to address this problem involve two representative lines of work. The first one relies on external execution, such as generating code or invoking tools to perform exact aggregation~\citep{chen2022program,tooloutputs2026,cao2026coding,logcopilot2026}. Although these methods can achieve exact aggregates, they require a separate generate--execute--return cycle, where pausing and restarting the model’s normal reasoning flow comes with additional resource consumption and time delay. 
The second one learns neural representations of sets, such as Deep Sets, Set Transformer, and Universal Mini-Batch Consistency (UMBC)~\citep{deepsets2017,settransformer2019,umbc2023,wagstaff2019}, that can be used to enhance long context aggregation. 
However, these general-purpose representations do not necessarily guarantee operator-specific behaviors, such as reliably recognizing repeated identities, nor do they provide predictable accuracy under a fixed resource budget. As will be demonstrated in our experiments, their aggregation error can increase as context length grows.
\textit{Contributions}: Motivated by these gaps, this paper aims to enhance the performance of long-context language models on non-additive, set-based aggregation tasks.
Accordingly, we introduce a model-side aggregation interface for long-context reasoning that maintains compact aggregation states alongside a frozen language model through HyperLogLog (HLL). The interface supports distinct counting, set combination, overlap estimation, and grouped aggregation with controlled error and bounded resource usage. Different from external execution methods, it does not require an additional generate--execute--return cycle. Different from general-purpose learned set representations, it provides explicit and predictable update, merge, and readout operations. More importantly, our method requires only lightweight state updates. For each input stream or active group, it maintains one fixed-budget sketch state, whose memory does not grow with context length or set cardinality.
Conceptually, this mechanism provides a form of approximate number sense to language models: it produces fast and sufficiently reliable estimates of quantities and set relationships and makes them directly available for downstream reasoning.

Our technical contributions include:
\begin{enumerate}
    \item We integrate lightweight, mergeable sketch states that are updated in parallel with the forward computation of a frozen language model, enabling efficient model-side aggregation without invoking an external executor.
    \item We develop a unified mergeable interface for set-based aggregation. It maintains one fixed-budget mergeable HLL state per input stream, from which distinct count, union, Jaccard similarity, and containment are read out. 

    \item We validate the proposed framework through both task-specific aggregation experiments and end-to-end reasoning evaluations, covering downstream reasoning performance, length extrapolation, mergeability, resource--error tradeoffs, and robustness to item-selection noise.
\end{enumerate}

It is worth noting that the HLL sketch algorithm used in the paper is not new. Our contribution lies in organizing aggregation states 
into a model-side interface for long-context reasoning, allowing them to be updated throughout context processing and their readouts to be made directly available for subsequent reasoning, and the validation of this design.

\section{Related Work}
\label{sec:related}
\subsection{Code Execution and Tool-Augmented Aggregation}

A common way to obtain exact aggregates is to introduce an external executor that computes these quantities for the language model. Program-Aided Language Models (PAL), Program of Thoughts (PoT), database-query systems, and coding agents could deploy aggregation computation to external Python runtimes or Structured Query Language (SQL) engines \citep{gao2023pal,chen2022program,cao2026coding}. 
LogCopilot applies the same idea to log analysis by translating natural-language requests into executable log queries \citep{logcopilot2026}. If the code is correctly generated and receives the relevant records, the executor can return an exact result to the language model as tool output \citep{tooloutputs2026}. However, the result becomes available only after query generation, execution, and tool return.
Although some executors can retain state across queries, an exact set still requires memory that grows with the number of distinct identities.
In contrast, our approach relies on a fixed-budget model-side state that is incorporated into the frozen model during context processing. Thus, the readouts are always immediately available for subsequent reasoning without invoking an external executor.

\subsection{Long-Context Aggregation and Learning-Based Set Representations}
Set representation and processing methods offer an alternative route for long-context set aggregation. Deep Sets~\citep{deepsets2017} maps items into permutation-invariant set functions, which can summarize variable-size sets without requiring the model to retain every item or depend on their input order.
Set Transformer~\citep{settransformer2019} uses attention to model interactions among set elements, allowing the aggregate to capture relationships that simple pooling may miss. It also reduces the cost of processing large sets.
Universal Mini-Batch Consistency (UMBC)~\citep{umbc2023} enables a large set to be processed in smaller partitions while guaranteeing that the results are invariant to different ways of set partitioning. This allows robust processing of extra-long contexts in a sequential manner. 
While these methods make set aggregation more scalable, the specific aggregation behaviors are still learned from data. They do not explicitly guarantee deterministic merging of independently computed states, nor offer a predictable error under a fixed resource budget. Our method, on the other hand, provides these properties through the design of an explicit aggregation state and update rules.
\subsection{Learned Data Structures and Streaming Sketches}
Beyond general set encoders, existing works also learn compact state representations for approximate set queries. Learned Bloom filters~\citep{rae2019meta,song2026elastic,vaidya2020partitioned} and neural sketches~\citep{cao2023meta} adapt their internal representations to the corresponding data distribution. Among them, MaxSketch is the closest method to our setting because it also maintains a compact, mergeable state for approximate cardinality estimation. However, MaxSketch is based on random Gaussian projections, which makes it suitable for noisy, high-dimensional observations~\citep{tsikouras2026maxsketch}. Our setting instead assumes that discrete item identities are directly available and focuses on integrating explicit streaming aggregation states with language-model reasoning.

From a technical perspective, our model-side interface leverages existing sketch operator HyperLogLog (HLL) to facilitate the update, merge, and value readout. Specifically, HLL estimates distinct cardinality with a fixed register array and merges two states by taking the maximum in each register \citep{hll2007}. In addition, paired HLL states can estimate Jaccard similarity and containment through joint maximum-likelihood estimation~\citep{ertl2017new}.
This allows a single compact state representation to support distinct counting, set combination, and overlap-related queries. There are alternative sketch operators that can provide similar capabilities. For example, $k$-Minimum Values (KMV) retains the $k$ smallest hash values and estimates both cardinality and overlap, whereas MinHash directly estimates Jaccard similarity~\citep{baryossef2002,minhash1997}. However, compared with HLL, their readouts are more specialized. For queries involving multiplicity, Count-Min Sketch can provide a complementary state for frequency estimation~\citep{countmin2004}.
We note that these sketch operators are established in the literature, our contribution lies in the interface that selects and maintains task-appropriate aggregation states alongside a frozen language model, and returns the resulting statistics to the model for downstream reasoning.

\section{Problem Setup}
\label{sec:prelim}
\label{sec:theory}
Consider a log-analysis request such as ``How many distinct users reported an error in each service?'' Identifying all error records is not sufficient to answer this question, because the same user may report multiple errors in the same service and should still be counted only once. Similarly, if the request instead asks how much the affected-user populations of two services overlap, knowing the number of users in each service is also insufficient: the computation must retain which users occur in both. Thus, after relevant records are identified, their identities may still need to be aggregated.

\smallskip
\noindent\textbf{Identity-Aware Aggregation.} 
Many non-additive aggregation tasks depend on how item identities occur across relevant records, including their uniqueness, repetition, grouping, and relationships across collections. We refer to this broad and fundamental family of tasks as \emph{identity-aware aggregation} and focus the scope of this paper on this setting. Specifically, we consider support-based tasks, which depend on which distinct identities occur and include distinct counting, union and intersection cardinalities, overlap, Jaccard similarity, and containment, as well as grouped requests that apply these operations separately within request-defined groups. These tasks capture a wide range of practical aggregation needs and are closely related to standard set operations in relational processing and established measures of set comparison~\citep{codd1970relational,minhash1997,setsketch2101}.

\smallskip
\noindent\textbf{Problem Formulation}.
To formalize our problem, let $x_{1:L}=(x_1,\ldots,x_L)$ denote a context of length $L$, and let $q$ denote a request over this context. 
We assume that an extractor associated with a frozen language model identifies the relevant records and organizes them into $\tau$ operand-specific streams (referred to as streams hereafter for simplicity):
\begin{align}
E(x_{1:L},q)
&=
(\mathcal R^{(1)},\ldots,\mathcal R^{(\tau)}), \nonumber\\
\text{where} ~\mathcal R^{(i)}
&=\bigl((z_t^{(i)},g_t^{(i)})\bigr)_{t=1}^{n_i}.
\label{eq:extractor}
\end{align}
Here, $\tau$ is the number of streams required by $q$ and $n_i$ is the number of records in stream $\mathcal R^{(i)}$. For each record, $z_t^{(i)}\in\mathcal U$ denotes its canonical identity in the item universe $\mathcal U$, and $g_t^{(i)}\in\mathcal G\cup\{\varnothing\}$ denotes an optional (i.e., $\varnothing$ if not used) group key.

For each stream $\mathcal R^{(i)}$, we maintain a fixed-budget \emph{sketch state} $s_i\in\mathcal S$, where 
$\mathcal S$ denotes the state space. Let
$y_q=f_q\!\left(
\mathcal R^{(1)},\ldots,\mathcal R^{(\tau)}
\right)$ denote the exact answer computed from the complete record streams.
A query-specific readout $\rho_q$ uses only the sketch states to compute an estimate:
\begin{equation}\label{eq_defrho}
\widehat y_q
=
\rho_q(s_1,\ldots,s_\tau)
\approx
y_q.
\end{equation}
The problem is therefore to construct and update the sketch states $s_i$ and to design the readout function $\rho_q$ that satisfies the following requirements:
\begin{itemize}
    \item \textbf{Model-side operation}: Each sketch state $s_i$ is updated as the model extracts records into the corresponding stream $\mathcal R^{(i)}$, with low computational overhead.
    
    \item \textbf{Fixed budget}: The sizes of the sketch states do not grow with the stream length or the number of distinct identities.
    
    \item \textbf{Aggregation consistency}: Sketch states constructed from different segments of the same stream must be mergeable into a state for the complete stream. 
\end{itemize}
Under these constraints, the objective is to allow $\widehat y_q$ to accurately approximate the exact answer $y_q$.

\smallskip
\begin{remark}
Normalized attention combines record representations through weighted averaging, which generally does not match the operations required by identity-aware aggregation. If an identity $z_t^{(i)}$ appears twice, the second record enters the average and usually changes the readout, whereas a support-based aggregation should remain unchanged. 
The mismatch also persists if a stream is processed in segments, where the same items in different segments are usually counted multiple times. Related issues have been identified in attention-based graph aggregation~\citep{zhang2020improving}. Appendix~A gives more detailed discussion and counterexamples.
\end{remark}

\section{Method}
\label{sec:method}

\begin{figure*}[t]
    \centering
    \includegraphics[width=\textwidth]{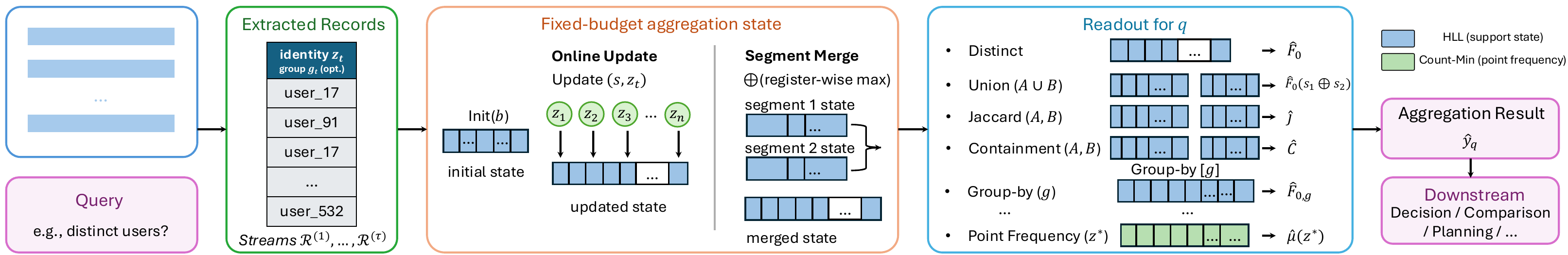}
\caption{
Model-side aggregation workflow. As the frozen language model processes the context, the extractor $E$ assigns canonical identities to relevant records and places them in the streams required by request $q$. For grouped requests, a group key routes each record to the corresponding group state. Each stream maintains a fixed-budget state, and states built from separate segments can be merged without reading the records again. The readout $\rho_q$ then obtains the requested aggregate from one or more states and returns $\widehat y_q$ to the model for downstream reasoning.
HLL is used for the set-based readouts, while Count--Min supports point-frequency queries.
}
    \label{fig:method}
\end{figure*}

Our method augments a frozen language model with explicit aggregation states implemented using sketch operators. Each state is assigned a user-defined budget that controls the tradeoff between estimation accuracy and storage overhead.
As the language model processes the context, the extractor $E$ identifies the relevant records, assigns each record a canonical identity $z_t^{(i)}$, and routes it to the corresponding stream $\mathcal R^{(i)}$. The associated sketch states run on model-side, i.e., they are updated online in parallel with the model’s forward computation and may be merged across compatible segments of the same stream.

Once all relevant records have been processed, a query-specific readout converts one or more states into an estimate. For example, distinct count estimates the number of identities from a single HyperLogLog (HLL) sketch state. Union size first merges two HLL states register by register and then reads the cardinality of the merged state. Jaccard similarity and containment keep the two sketch states separate; Joint maximum-likelihood estimation (JMLE) jointly estimates the distinct numbers of identities appearing only in either stream and in both streams, from which the two ratios are computed. Group-by applies the same readouts to the HLL state of the requested group. These queries reuse the maintained HLL states rather than constructing a separate sketch for each readout. 
Finally, the readout $\rho_q$ produces $\widehat y_q$, which is appended to the model input as explicit evidence, and decoding continues from the retained KV cache. Figure~\ref{fig:method} illustrates this data flow.

\subsection{State construction from streams}
\label{sec:state_construction}
We first construct a fixed-size state that records the distinct identities in each stream. For this purpose, we use HyperLogLog (HLL)~\citep{hll2007}. Stream $i$ maintains the state
$s_i=(M_{i,1},\ldots,M_{i,m})$,
where $m$ is the number of registers (size) in the state and all registers are initialized to zero.
A fixed hash function maps each identity $z$ to a register index $j(z)$ and a rank $r(z)$:
$\operatorname{Hash}(z)=(j(z),r(z))$. Only the register selected by $j(z)$ is updated:
\begin{equation}
M_{i,j(z)}
\leftarrow
\max\!\left\{M_{i,j(z)},r(z)\right\}.
\label{eq:hll_update}
\end{equation}
Because the hash function is fixed, every occurrence of the same identity produces the same pair $(j(z),r(z))$. Once the selected register has reached $r(z)$, other occurrences of $z$ have no further effect. Processing all identities in $\mathcal R^{(i)}$ in this way produces the final state $s_i$.

For grouped requests, we first use the group key $g_t^{(i)}$ to select the state associated with that group, then, $z_t^{(i)}$ updates one of its registers according to~\eqref{eq:hll_update}. A repeated identity therefore has no additional effect within the same group, but an occurrence in another group updates that group's state independently.

A stream can also be processed in $P_i$ segments. Let
$s_i^{[a]}=(M_{i,1}^{[a]},\ldots,M_{i,m}^{[a]})$
be the state produced from segment $a$. When all segment states use the same register count, hash function, and seed, the state of the full stream is obtained by taking the maximum in each register:
\begin{equation}
s_i=\bigoplus_{a=1}^{P_i}s_i^{[a]},
\qquad
M_{i,j}
=
\max_{1\leq a\leq P_i}M_{i,j}^{[a]}.
\label{eq:operand_segment_merge}
\end{equation}
The merge produces the same HLL state compared with processing the complete stream as a whole. States associated with different streams or groups remain separate because later readouts may need to distinguish where each identity occurred.

The HLL construction above supports uniqueness-based queries, which are the primary focus of this paper. However, we note that hash-based methods can also handle frequency (multiplicity) queries that accumulate repeated occurrences. For this purpose, we use Count--Min Sketch~\citep{countmin2004}, which maintains a fixed-size hash table and estimates an identity's frequency from its associated counters. Compatible Count--Min states can be merged by elementwise addition. The detailed update and readout rules are provided in Appendix~D.

\subsection{State reuse and query readout}
\label{sec:hll_reuse}

Once the HLL state of a stream has been constructed, it can be reused for multiple set queries without rebuilding the state. Let $\widehat F_0(s)$ denote the specific standard HLL cardinality estimator corresponding to the $\rho_q$ in \eqref{eq_defrho}. A distinct-count query reads the cardinality directly from one state, whereas a union query first forms a temporary registerwise merge and then reads its cardinality:
\begin{equation}
\begin{aligned}
\widehat y_q
=\rho_q(s_i)
&=\widehat F_0(s_i),
&& q\text{ asks for distinct count},\\
\widehat y_q
=\rho_q(s_1,s_2)
&=\widehat F_0(s_1\oplus s_2),
&& q\text{ asks for union}.
\end{aligned}
\label{eq:hll_cardinality_readouts}
\end{equation}
The union merge is temporary and does not overwrite $s_1$ or $s_2$. Their separate states remain available for other readouts.

For Jaccard similarity and containment, we use the HLL joint maximum-likelihood estimator (JMLE)~\citep{ertl2017new}. Let $s_1$ and $s_2$ represent sets $A$ and $B$, respectively, and define
$n_{10}=|A\setminus B|$,
$n_{01}=|B\setminus A|$, and
$n_{11}=|A\cap B|$.
JMLE estimates these three disjoint regions. The corresponding readouts are
\begin{equation}
\begin{aligned}
\widehat y_q
&=\rho_q(s_1,s_2)\\
&=
\begin{cases}
\dfrac{\widehat n_{11}}
{\widehat n_{10}+\widehat n_{01}+\widehat n_{11}},
& \text{Jaccard similarity},\\[8pt]
\dfrac{\widehat n_{11}}
{\widehat n_{10}+\widehat n_{11}},
& \text{containment of }A\text{ in }B.
\end{cases}
\end{aligned}
\label{eq:hll_relation_readouts}
\end{equation}

Union continues to use the registerwise merge in Equation~\eqref{eq:hll_cardinality_readouts}. If the JMLE optimizer does not return finite, feasible estimates that satisfy its convergence checks, the relation readout is marked invalid rather than replaced with an inclusion--exclusion estimate based on
$|A\cap B|=|A|+|B|-|A\cup B|$.
Appendix~C.4 provides details of the likelihood and optimization procedure.

Because the sketch states are preserved after construction, distinct count, union size, Jaccard similarity, and containment can all be obtained from the same HLL registers by changing only the readout. Group-by applies the same reuse within each selected group. An existing state may also be reused for a later request, provided that its definition, identity mapping, register count, hash function, and seed remain unchanged.

\noindent \textbf{Complexity.}
Because our approach performs aggregation on the model side, we briefly summarize the storage and computational complexity of the HLL states. For storage, our implementation represents each HLL register using one byte. Thus, an HLL state with $m=2048$, (i.e., the number of registers in each state) occupies $2$ KiB. For a fixed number of streams and active groups, processing additional records only updates the existing states and does not increase their size. Each new stream or active group requires one additional HLL state, while group-by queries also require an index over the active group keys.

For computation, each HLL update modifies a single register and therefore takes $O(1)$ time. Registerwise state merging and cardinality estimation each require $O(m)$ ($m$ is fixed by the state budget). A JMLE readout with $I$ optimization steps takes $O(mI)$ time.

\section{Experiments}
\label{sec:experiments}
\label{sec:results}

\begin{figure*}[!tb]
\centering
\includegraphics[width=.8\textwidth]{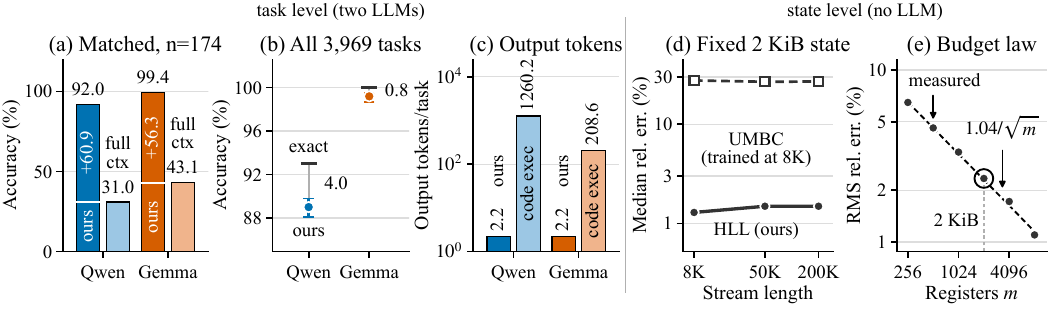}
\caption{Main results. Colour is the model (\textcolor[HTML]{0072B2}{\textbf{blue}} Qwen 3.6 (35B), \textcolor[HTML]{D55E00}{\textbf{vermillion}} Gemma 4 (31B)); solid fill is SketchOps, pale fill the comparison method. \textbf{(a)} 174 matched tasks against the strongest full-context baseline. \textbf{(b)} All 3,969 tasks: the black tick is exact aggregation and the whisker the 95\% window-cluster CI of the paired gap. \textbf{(c)} Mean output tokens per task against external code execution. \textbf{(d,\,e)} State level, no language model: relative error stays flat in stream length under a fixed 2\,KiB budget and follows $1.04/\sqrt{m}$.}
\label{fig:main}
\end{figure*}

This section evaluates whether fixed-budget aggregation states can preserve statistics from long record streams and subsequently support downstream reasoning. We validate our method on scenario-based \emph{aggregate-then-reason} tasks, which require the model to obtain the relevant aggregate from the records and then use it as intermediate evidence to make a final decision. We compare our method with exact aggregation, full-context reasoning, and code execution. Additionally, to examine the effectiveness of our method at the component level, we isolate the behavior of the aggregation states themselves through experiments that measure estimation error, merge consistency, the resource--error trade-off, and computational overhead. Figure~\ref{fig:main} summarizes the main results from these evaluations.

\subsection{Evaluation setup}
\label{sec:evaluation_setup}

The reasoning experiments use the publicly available Oolong-Synth test set~\citep{bertsch2025oolong}. This dataset organizes news articles, user comments, question answering, natural-language inference, and various linguistic data into long record sequences,  with user identities and category labels provided for each record. We selected 174 record sequences from this dataset, referring to each sequence as a source window. These windows have lengths ranging from 1K to 4M tokens.

Based on the labels provided by the dataset, we converted each source window into a structured identity stream, then constructed multiple \emph{aggregate-then-reason} tasks from the same records. The required aggregates include distinct counts, grouped distinct counts, set relationships, and a small number of point-frequency queries. Each task also specifies a decision condition that maps the aggregate to one of two candidate answers. Under this setup, the aggregate serves only as intermediate evidence, and the model must still make a final selection based on it. For example, the model selects between ``Handle Immediately'' and ``Continue Monitoring'', depending on whether more than 100 users report an error. Because multiple tasks were constructed from each source window, the 174 windows yielded 3,969 tasks in total. These tasks were used to compare fixed-budget aggregation states with exact aggregation and to conduct a counterfactual evidence-use test. We evaluated Qwen 3.6 (35B) and Gemma 4 (31B) using the same tasks, aggregation states, prompts, and scoring rules.

Accuracy for all reasoning tasks was calculated based on whether the model's final selection matched the reference answer. Model-call failures, context overflows, and unparseable answers were counted as errors. 
Because tasks generated from the same source window share the underlying records, they cannot be considered independent of one another. Thus, we treated a source window and its derived tasks as a single resampling unit when computing confidence intervals. We repeated this resampling procedure 10,000 times and calculated 95\% confidence intervals from the resulting distribution. When presenting results, we refer to our method as \textsc{SketchOps}.

\subsection{Explicit aggregation states improve downstream reasoning}

The full-context reasoning comparison used 174 matched tasks. Since a single source window can generate multiple related questions, we selected one question from each window according to a predetermined, answer-independent rule, ensuring that each underlying record sequence appeared only once in the comparison. The full-context baselines directly read the original records and used direct answering~\cite{kojima2022large}, chain-of-thought (CoT) reasoning~\cite{wei2022chain}, or guided-choice prompting~\cite{geng2023grammar}. Our method first processes the same records with fixed-budget aggregation states, then passes the resulting aggregate to the model. All methods receive the same questions and candidate answers. Tasks exceeding the 150K-token context limit remain in the evaluation and are scored as incorrect.

Table~\ref{tab:matched_state_value} reports these results. On Qwen, our method achieves 92.0\%, which is 60.9 percentage points higher than the strongest full-context baseline, CoT. On Gemma, our method achieves 99.4\%, exceeding the best full-context result by 56.3 points. Part of this difference comes from context length. Of the 174 tasks, 80 exceed the 150K-token limit and are therefore scored as incorrect for the full-context baselines, while our method still achieves 86.2\% on Qwen and 98.8\% on Gemma for these tasks. More importantly, the advantage remains among the 94 tasks that fit within the context limit. Under this subset, our method reaches 96.8\% and 100\%, whereas the strongest full-context baseline reaches 57.4\% and 79.8\%, respectively. These results suggest two benefits of the proposed aggregation state. First, it preserves the required aggregate when the records exceed the model's context limit. Second, it gives the model explicit results for deduplication, grouping, and set relations when the records do fit.

\begin{table}[t]
\centering
\small
\setlength{\tabcolsep}{4pt}
\renewcommand{\arraystretch}{1.08}
\caption{Accuracy on 174 matched aggregate-then-reason tasks. One task is selected from each source window, and all methods use the same question and answer choices.}
\label{tab:matched_state_value}
\begin{tabular}{
@{}
L{0.48\columnwidth}
C{0.20\columnwidth}
C{0.20\columnwidth}
@{}}
\toprule
Method & Qwen (\%) & Gemma (\%) \\
\midrule

\textsc{SketchOps} (ours) &
\textbf{92.0} &
\textbf{99.4} \\

Direct full-context reasoning &
28.7 &
43.1 \\

Full context $+$ CoT &
31.0 &
36.2 \\

Full context $+$ guided choice &
27.6 &
43.1 \\

\bottomrule
\end{tabular}
\end{table}

We next evaluate our method on all 3,969 tasks and measure the accuracy degradation introduced by the fixed-budget state. The exact-aggregation condition in Table~\ref{tab:large_state_results} follows the same pipeline as our method but replaces the fixed-budget readout with an exact aggregate computed from the complete streams. On Qwen, our method achieves 89.0\%, while exact aggregation achieves 93.0\%, yielding a gap of 4.0 percentage points (95\% window-cluster CI: 3.2--4.9 points). On Gemma 4 (31B, BF16), the corresponding accuracies are 99.2\% and 100.0\%, with a gap of 0.8 percentage points (95\% window-cluster CI: 0.5--1.3 percentage points). This comparison measures the downstream accuracy lost through fixed-budget approximation. The preceding matched experiment separately measures the value of supplying an explicit aggregation state rather than asking the model to recover the same information from the full records. Overall, the fixed-budget state preserves most of the downstream performance of exact aggregation while retaining bounded memory usage and model-side operation.

Finally, we test whether the model's decision output responds to the supplied aggregate. For each task, we replace the aggregate with a counterfactual aggregate that changes the correct answer while leaving the decision condition unchanged. The model selects the outcome implied by the modified aggregate on 92.8\% of the Qwen tasks and 100.0\% of the Gemma tasks. In most cases, the downstream decision changes, which confirms that the model generally uses the supplied aggregate as evidence when making its final decision.
\begin{table}[t]
\centering
\small
\setlength{\tabcolsep}{3.5pt}
\renewcommand{\arraystretch}{1.08}
\caption{Results on 3,969 aggregate-then-reason tasks. The exact-aggregate row replaces the fixed-budget readout with the exact value. The counterfactual row replaces the aggregate with a value that changes the correct answer and is scored against the outcome implied by the modified evidence.}
\label{tab:large_state_results}
\begin{tabular}{
@{}
L{0.48\columnwidth}
C{0.20\columnwidth}
C{0.20\columnwidth}
@{}}
\toprule
Condition & Qwen (\%) & Gemma (\%) \\
\midrule

\multicolumn{3}{@{}l}{\emph{Approximation comparison}} \\

\textsc{SketchOps} (ours) &
89.0 &
99.2 \\

\quad with exact aggregate &
93.0 &
100.0 \\

\addlinespace
\multicolumn{3}{@{}l}{\emph{Evidence-use test}} \\

\quad with counterfactual &
92.8 &
100.0 \\

\bottomrule
\end{tabular}
\end{table}

\subsection{Comparison with external code execution}

We also compared \textsc{SketchOps} (ours) with exact code execution~\cite{gao2023pal,chen2022program} using 1,200 tasks preselected from the set of 3,969 tasks, 
which covered all 174 source windows and different aggregation types. Qwen and Gemma were evaluated using the same tasks and experimental protocol. In the code-execution path, the model first generates a complete Python program, which is then run by a controlled executor that returns the result. With our method, the same model directly receives the fixed-budget state readout and generates only the final selection. Table~\ref{tab:code_execution_comparison} reports accuracy and generated output tokens for the two approaches. An independent CPU verifier re-derives every code-execution answer from the original records.

\begin{table}[t]
\centering
\small
\setlength{\tabcolsep}{2.5pt}
\renewcommand{\arraystretch}{1.08}
\caption{Comparison of \textsc{SketchOps} (ours) and external code execution on the same 1,200 tasks for each model. Output tokens are reported as the mean per task.}
\label{tab:code_execution_comparison}
\begin{tabular}{
@{}
L{0.16\columnwidth}
L{0.34\columnwidth}
C{0.16\columnwidth}
C{0.22\columnwidth}
@{}}
\toprule
Model & Method & Acc. (\%) & Tokens / task \\
\midrule

Qwen &
\textsc{SketchOps} &
91.1 & 2.2 \\

Qwen &
Code execution &
100.0 & 1,260.2 \\

\addlinespace

Gemma &
\textsc{SketchOps} &
99.3 & 2.2 \\

Gemma &
Code execution &
100.0 & 208.6 \\

\bottomrule
\end{tabular}
\end{table}
On Qwen, code execution exceeded \textsc{SketchOps} (ours) by 8.9 percentage points (95\% source-window-cluster bootstrap CI: 7.1--10.7 percentage points); on Gemma, the difference was less than one percentage point. This accuracy advantage came with a much longer code generation path. Code execution produced an average of 1,260.2 output tokens on Qwen and 208.6 on Gemma, compared with 2.2 for \textsc{SketchOps} on both models. 
\textsc{SketchOps} generated only the final decision and did not generate or execute a program. While external code execution is reliable when exactness of the result matters most, our \textsc{SketchOps} is intended for a different setting, where an aggregation state must be updated and merged within a fixed memory budget and its readout must be available without generating a program.

\subsection{Length extrapolation and resource--error trade-off}
\label{sec:length_budget}

The preceding downstream results rely on the aggregation state providing reliable statistics. To test this in an isolated environment, we fix the identity streams and exact targets and compare only the state-readout errors. In the length-extrapolation experiment, both methods use a $2$ KiB carried state for each record stream. UMBC~\cite{umbc2023} is trained on record streams of length $8$K. HLL requires no training. During testing, UMBC and HLL process the same identity streams, with the longest stream containing $200$K records. The number of distinct identities in each stream is sampled independently of stream length from $1{,}000$ to $4{,}000$. Therefore, neither method can infer the exact distinct count from the number of records alone.

\begin{table}[t]
\centering
\small
\setlength{\tabcolsep}{5pt}
\caption{Length extrapolation under a fixed $2$ KiB carried-state budget. Each length contains 100 test streams; entries are median relative errors for distinct count. UMBC's static model parameters are excluded from the per-stream state budget.}
\label{tab:length_extrapolation}
\begin{tabular}{@{}lccc@{}}
\toprule
Stream length $L$ & HLL & Official UMBC & UMBC/HLL \\
\midrule
$8$K   & \textbf{1.29\%} & 27.60\% & $21.4\times$ \\
$50$K  & \textbf{1.49\%} & 26.78\% & $18.0\times$ \\
$200$K & \textbf{1.49\%} & 26.95\% & $18.1\times$ \\
\bottomrule
\end{tabular}
\end{table}

Table~\ref{tab:length_extrapolation} shows that, as the number of records increased from $8$K to $200$K, HLL's median relative error rose only from 1.29\% to 1.49\%, while its state remained fixed at $2$ KiB. Under the same carried-state budget, UMBC's error remained close to 27\%. This comparison isolates the aggregation state used for distinct count. It shows that an update rule matched to set semantics can track distinct cardinality without learning this behavior from training data.

A fixed state size also comes at a cost in accuracy. We therefore varied the number of HLL registers $m$ and measured the root-mean-square (RMS) relative error over a grid in which the number of distinct identities ranged from $100$ to $1{,}000{,}000$. The current implementation allocates one byte per register, so the state size in bytes corresponds directly to $m$. We also tested whether segmenting a stream changes the HLL result. After dividing each identity stream into as many as 256 segments and merging the resulting states in different orders, we found no register mismatches or readout differences from a single pass over the full stream.

\begin{table}[t]
\centering
\small
\setlength{\tabcolsep}{5pt}
\caption{Resource--error trade-off for HLL. The measured column reports RMS relative error; the final column gives the reference value $1.04/\sqrt{m}$.}
\label{tab:budget_tradeoff}
\begin{tabular}{@{}rccc@{}}
\toprule
Registers $m$ & State size & Measured error & Reference \\
\midrule
$256$     & $0.25$ KiB & 6.47\% & 6.50\% \\
$512$     & $0.50$ KiB & 4.60\% & 4.60\% \\
$1{,}024$ & $1$ KiB    & 3.33\% & 3.25\% \\
$2{,}048$ & $2$ KiB    & 2.34\% & 2.30\% \\
$4{,}096$ & $4$ KiB    & 1.72\% & 1.63\% \\
$8{,}192$ & $8$ KiB    & 1.10\% & 1.15\% \\
\bottomrule
\end{tabular}
\end{table}

Table~\ref{tab:budget_tradeoff} shows that the measured error generally follows the reference value $1.04/\sqrt{m}$, the standard asymptotic relative error for HLL. This error is caused by the approximation using a fixed register budget that does not expand with additional stream length. 
The storage crossover between a 2-KiB HLL sketch state and an exact identity set is reported in Appendix~F.3.

\section{Conclusion}
\label{sec:conclusion}

We introduce a model-side aggregation interface for long-context reasoning that maintains compact aggregation states alongside a frozen language model through HyperLogLog (HLL). The interface supports distinct counting, set combination, overlap estimation, and grouped aggregation with controlled error and bounded resource usage. Different from external execution methods, it does not require an additional generate--execute--return cycle. For each stream or active group, it maintains one fixed-budget sketch state, whose memory does not grow with context length or set cardinality.
Through experiment validation using Qwen and Gemma, we show that our method substantially outperforms direct full-context reasoning and chain-of-thought (CoT) reasoning. The counterfactual test further confirms that the models change their final choices in response to the supplied aggregate. State-level experiments show that state size remains fixed as the record stream grows, while states constructed from separate segments merge into exactly the same state produced by a single pass over the full stream.

\noindent\textbf{Limitations:} Although each sketch operator uses a fixed-size state, it introduces a trade-off between estimation error for bounded state size. Dedicated storage or external execution might still be preferred for small-scale tasks that require exact numerical outputs.
The fixed state size applies to a single stream or active group, so total aggregation memory still grows with the number of states maintained. 
Also, the current interface supports only the aggregation targets considered here: distinct count, union size, Jaccard similarity, containment, grouped distinct count, and point frequency. Additional targets require corresponding state implementations and readout functions.

Our future work involves dynamically selecting between exact sets and sketches with different budgets based on set cardinality, error requirements, and available resources. We will also extend the state and readout interfaces to additional aggregation operators. 

\bibliography{refs}

\clearpage
\onecolumn
\providecommand{\ATRClaimField}[2]{%
  \ifcsname ATRClaim@#1@#2\endcsname
    \csname ATRClaim@#1@#2\endcsname
  \else
    \PackageError{atr-paper-results}{Unknown claim field `#1/#2'}{}%
  \fi}
\providecommand{\ATRClaimValue}[1]{\ATRClaimField{#1}{value}}
\providecommand{\ATRClaimN}[1]{\ATRClaimField{#1}{n}}
\providecommand{\ATRClaimWindows}[1]{\ATRClaimField{#1}{windows}}
\providecommand{\ATRClaimCILow}[1]{\ATRClaimField{#1}{ciLow}}
\providecommand{\ATRClaimCIHigh}[1]{\ATRClaimField{#1}{ciHigh}}
\providecommand{\ATRClaimDiagnostics}[1]{\ATRClaimField{#1}{diagnostics}}
\providecommand{\ATRClaimDiagnostic}[2]{\ATRClaimField{#1}{diagnostic@#2}}
\providecommand{\ATRClaimUnit}[1]{\ATRClaimField{#1}{unit}}
\providecommand{\ATRClaimScope}[1]{\ATRClaimField{#1}{scope}}
\providecommand{\ATRClaimCount}{69}
\expandafter\def\csname ATRClaim@cross.gemma.budgeted_minus_cot@value\endcsname{0.632183908045977}
\expandafter\def\csname ATRClaim@cross.gemma.budgeted_minus_cot@n\endcsname{174}
\expandafter\def\csname ATRClaim@cross.gemma.budgeted_minus_cot@windows\endcsname{174}
\expandafter\def\csname ATRClaim@cross.gemma.budgeted_minus_cot@ciLow\endcsname{0.5632183908045977}
\expandafter\def\csname ATRClaim@cross.gemma.budgeted_minus_cot@ciHigh\endcsname{0.7011494252873564}
\expandafter\def\csname ATRClaim@cross.gemma.budgeted_minus_cot@diagnostics\endcsname{\{\}}
\expandafter\def\csname ATRClaim@cross.gemma.budgeted_minus_cot@unit\endcsname{null}
\expandafter\def\csname ATRClaim@cross.gemma.budgeted_minus_cot@scope\endcsname{null}
\expandafter\def\csname ATRClaim@cross.gemma.budgeted_minus_direct@value\endcsname{0.5632183908045977}
\expandafter\def\csname ATRClaim@cross.gemma.budgeted_minus_direct@n\endcsname{174}
\expandafter\def\csname ATRClaim@cross.gemma.budgeted_minus_direct@windows\endcsname{174}
\expandafter\def\csname ATRClaim@cross.gemma.budgeted_minus_direct@ciLow\endcsname{0.4885057471264368}
\expandafter\def\csname ATRClaim@cross.gemma.budgeted_minus_direct@ciHigh\endcsname{0.6379310344827587}
\expandafter\def\csname ATRClaim@cross.gemma.budgeted_minus_direct@diagnostics\endcsname{\{\}}
\expandafter\def\csname ATRClaim@cross.gemma.budgeted_minus_direct@unit\endcsname{null}
\expandafter\def\csname ATRClaim@cross.gemma.budgeted_minus_direct@scope\endcsname{null}
\expandafter\def\csname ATRClaim@cross.gemma.budgeted_minus_guided@value\endcsname{0.5632183908045977}
\expandafter\def\csname ATRClaim@cross.gemma.budgeted_minus_guided@n\endcsname{174}
\expandafter\def\csname ATRClaim@cross.gemma.budgeted_minus_guided@windows\endcsname{174}
\expandafter\def\csname ATRClaim@cross.gemma.budgeted_minus_guided@ciLow\endcsname{0.4885057471264368}
\expandafter\def\csname ATRClaim@cross.gemma.budgeted_minus_guided@ciHigh\endcsname{0.6379310344827587}
\expandafter\def\csname ATRClaim@cross.gemma.budgeted_minus_guided@diagnostics\endcsname{\{\}}
\expandafter\def\csname ATRClaim@cross.gemma.budgeted_minus_guided@unit\endcsname{null}
\expandafter\def\csname ATRClaim@cross.gemma.budgeted_minus_guided@scope\endcsname{null}
\expandafter\def\csname ATRClaim@cross.gemma.commoneligible_budgeted@value\endcsname{1.0}
\expandafter\def\csname ATRClaim@cross.gemma.commoneligible_budgeted@n\endcsname{94}
\expandafter\def\csname ATRClaim@cross.gemma.commoneligible_budgeted@windows\endcsname{94}
\expandafter\def\csname ATRClaim@cross.gemma.commoneligible_budgeted@ciLow\endcsname{1.0}
\expandafter\def\csname ATRClaim@cross.gemma.commoneligible_budgeted@ciHigh\endcsname{1.0}
\expandafter\def\csname ATRClaim@cross.gemma.commoneligible_budgeted@diagnostics\endcsname{\{"failure\_stage":\{"none":94\},"n\_context\_ineligible":0,"n\_parse\_failed":0,"n\_transport\_or\_harness\_failed":0,"n\_truncated\_at\_cap":0\}}
\expandafter\def\csname ATRClaim@cross.gemma.commoneligible_budgeted@unit\endcsname{null}
\expandafter\def\csname ATRClaim@cross.gemma.commoneligible_budgeted@scope\endcsname{common\_context\_eligible\_sensitivity}
\expandafter\def\csname ATRClaim@cross.gemma.commoneligible_budgeted@diagnostic@failure_stage\endcsname{\{"none":94\}}
\expandafter\def\csname ATRClaim@cross.gemma.commoneligible_budgeted@diagnostic@n_context_ineligible\endcsname{0}
\expandafter\def\csname ATRClaim@cross.gemma.commoneligible_budgeted@diagnostic@n_parse_failed\endcsname{0}
\expandafter\def\csname ATRClaim@cross.gemma.commoneligible_budgeted@diagnostic@n_transport_or_harness_failed\endcsname{0}
\expandafter\def\csname ATRClaim@cross.gemma.commoneligible_budgeted@diagnostic@n_truncated_at_cap\endcsname{0}
\expandafter\def\csname ATRClaim@cross.gemma.commoneligible_budgeted_minus_guided@value\endcsname{0.20212765957446807}
\expandafter\def\csname ATRClaim@cross.gemma.commoneligible_budgeted_minus_guided@n\endcsname{94}
\expandafter\def\csname ATRClaim@cross.gemma.commoneligible_budgeted_minus_guided@windows\endcsname{94}
\expandafter\def\csname ATRClaim@cross.gemma.commoneligible_budgeted_minus_guided@ciLow\endcsname{0.1276595744680851}
\expandafter\def\csname ATRClaim@cross.gemma.commoneligible_budgeted_minus_guided@ciHigh\endcsname{0.2872340425531915}
\expandafter\def\csname ATRClaim@cross.gemma.commoneligible_budgeted_minus_guided@diagnostics\endcsname{\{\}}
\expandafter\def\csname ATRClaim@cross.gemma.commoneligible_budgeted_minus_guided@unit\endcsname{null}
\expandafter\def\csname ATRClaim@cross.gemma.commoneligible_budgeted_minus_guided@scope\endcsname{common\_context\_eligible\_sensitivity}
\expandafter\def\csname ATRClaim@cross.gemma.commoneligible_guided@value\endcsname{0.7978723404255319}
\expandafter\def\csname ATRClaim@cross.gemma.commoneligible_guided@n\endcsname{94}
\expandafter\def\csname ATRClaim@cross.gemma.commoneligible_guided@windows\endcsname{94}
\expandafter\def\csname ATRClaim@cross.gemma.commoneligible_guided@ciLow\endcsname{0.7127659574468085}
\expandafter\def\csname ATRClaim@cross.gemma.commoneligible_guided@ciHigh\endcsname{0.8723404255319149}
\expandafter\def\csname ATRClaim@cross.gemma.commoneligible_guided@diagnostics\endcsname{\{"failure\_stage":\{"none":94\},"n\_context\_ineligible":0,"n\_parse\_failed":0,"n\_transport\_or\_harness\_failed":0,"n\_truncated\_at\_cap":0\}}
\expandafter\def\csname ATRClaim@cross.gemma.commoneligible_guided@unit\endcsname{null}
\expandafter\def\csname ATRClaim@cross.gemma.commoneligible_guided@scope\endcsname{common\_context\_eligible\_sensitivity}
\expandafter\def\csname ATRClaim@cross.gemma.commoneligible_guided@diagnostic@failure_stage\endcsname{\{"none":94\}}
\expandafter\def\csname ATRClaim@cross.gemma.commoneligible_guided@diagnostic@n_context_ineligible\endcsname{0}
\expandafter\def\csname ATRClaim@cross.gemma.commoneligible_guided@diagnostic@n_parse_failed\endcsname{0}
\expandafter\def\csname ATRClaim@cross.gemma.commoneligible_guided@diagnostic@n_transport_or_harness_failed\endcsname{0}
\expandafter\def\csname ATRClaim@cross.gemma.commoneligible_guided@diagnostic@n_truncated_at_cap\endcsname{0}
\expandafter\def\csname ATRClaim@cross.gemma.cot_context@value\endcsname{0.3620689655172414}
\expandafter\def\csname ATRClaim@cross.gemma.cot_context@n\endcsname{174}
\expandafter\def\csname ATRClaim@cross.gemma.cot_context@windows\endcsname{174}
\expandafter\def\csname ATRClaim@cross.gemma.cot_context@ciLow\endcsname{0.29310344827586204}
\expandafter\def\csname ATRClaim@cross.gemma.cot_context@ciHigh\endcsname{0.43103448275862066}
\expandafter\def\csname ATRClaim@cross.gemma.cot_context@diagnostics\endcsname{\{"failure\_as\_zero":true,"failure\_by\_stage":\{"context\_packing":80,"parsing":28\},"n\_failed\_flag\_true":80,"n\_failure\_events":108,"n\_parse\_failed":28\}}
\expandafter\def\csname ATRClaim@cross.gemma.cot_context@unit\endcsname{null}
\expandafter\def\csname ATRClaim@cross.gemma.cot_context@scope\endcsname{null}
\expandafter\def\csname ATRClaim@cross.gemma.cot_context@diagnostic@failure_as_zero\endcsname{true}
\expandafter\def\csname ATRClaim@cross.gemma.cot_context@diagnostic@failure_by_stage\endcsname{\{"context\_packing":80,"parsing":28\}}
\expandafter\def\csname ATRClaim@cross.gemma.cot_context@diagnostic@n_failed_flag_true\endcsname{80}
\expandafter\def\csname ATRClaim@cross.gemma.cot_context@diagnostic@n_failure_events\endcsname{108}
\expandafter\def\csname ATRClaim@cross.gemma.cot_context@diagnostic@n_parse_failed\endcsname{28}
\expandafter\def\csname ATRClaim@cross.gemma.direct_context@value\endcsname{0.43103448275862066}
\expandafter\def\csname ATRClaim@cross.gemma.direct_context@n\endcsname{174}
\expandafter\def\csname ATRClaim@cross.gemma.direct_context@windows\endcsname{174}
\expandafter\def\csname ATRClaim@cross.gemma.direct_context@ciLow\endcsname{0.3563218390804598}
\expandafter\def\csname ATRClaim@cross.gemma.direct_context@ciHigh\endcsname{0.5057471264367817}
\expandafter\def\csname ATRClaim@cross.gemma.direct_context@diagnostics\endcsname{\{"failure\_as\_zero":true,"failure\_by\_stage":\{"context\_packing":80\},"n\_failed\_flag\_true":80,"n\_failure\_events":80,"n\_parse\_failed":0\}}
\expandafter\def\csname ATRClaim@cross.gemma.direct_context@unit\endcsname{null}
\expandafter\def\csname ATRClaim@cross.gemma.direct_context@scope\endcsname{null}
\expandafter\def\csname ATRClaim@cross.gemma.direct_context@diagnostic@failure_as_zero\endcsname{true}
\expandafter\def\csname ATRClaim@cross.gemma.direct_context@diagnostic@failure_by_stage\endcsname{\{"context\_packing":80\}}
\expandafter\def\csname ATRClaim@cross.gemma.direct_context@diagnostic@n_failed_flag_true\endcsname{80}
\expandafter\def\csname ATRClaim@cross.gemma.direct_context@diagnostic@n_failure_events\endcsname{80}
\expandafter\def\csname ATRClaim@cross.gemma.direct_context@diagnostic@n_parse_failed\endcsname{0}
\expandafter\def\csname ATRClaim@cross.gemma.full_method_aggregate_only@value\endcsname{0.5083333333333333}
\expandafter\def\csname ATRClaim@cross.gemma.full_method_aggregate_only@n\endcsname{1200}
\expandafter\def\csname ATRClaim@cross.gemma.full_method_aggregate_only@windows\endcsname{174}
\expandafter\def\csname ATRClaim@cross.gemma.full_method_aggregate_only@ciLow\endcsname{0.48218724109362054}
\expandafter\def\csname ATRClaim@cross.gemma.full_method_aggregate_only@ciHigh\endcsname{0.5344398340248963}
\expandafter\def\csname ATRClaim@cross.gemma.full_method_aggregate_only@diagnostics\endcsname{\{"failure\_as\_zero":true,"failure\_by\_stage":\{\},"n\_failed\_flag\_true":0,"n\_failure\_events":0,"n\_parse\_failed":0\}}
\expandafter\def\csname ATRClaim@cross.gemma.full_method_aggregate_only@unit\endcsname{null}
\expandafter\def\csname ATRClaim@cross.gemma.full_method_aggregate_only@scope\endcsname{standalone\_full\_method}
\expandafter\def\csname ATRClaim@cross.gemma.full_method_aggregate_only@diagnostic@failure_as_zero\endcsname{true}
\expandafter\def\csname ATRClaim@cross.gemma.full_method_aggregate_only@diagnostic@failure_by_stage\endcsname{\{\}}
\expandafter\def\csname ATRClaim@cross.gemma.full_method_aggregate_only@diagnostic@n_failed_flag_true\endcsname{0}
\expandafter\def\csname ATRClaim@cross.gemma.full_method_aggregate_only@diagnostic@n_failure_events\endcsname{0}
\expandafter\def\csname ATRClaim@cross.gemma.full_method_aggregate_only@diagnostic@n_parse_failed\endcsname{0}
\expandafter\def\csname ATRClaim@cross.gemma.full_method_budgeted@value\endcsname{0.9925}
\expandafter\def\csname ATRClaim@cross.gemma.full_method_budgeted@n\endcsname{1200}
\expandafter\def\csname ATRClaim@cross.gemma.full_method_budgeted@windows\endcsname{174}
\expandafter\def\csname ATRClaim@cross.gemma.full_method_budgeted@ciLow\endcsname{0.9866999168744804}
\expandafter\def\csname ATRClaim@cross.gemma.full_method_budgeted@ciHigh\endcsname{0.9974916387959866}
\expandafter\def\csname ATRClaim@cross.gemma.full_method_budgeted@diagnostics\endcsname{\{"failure\_as\_zero":true,"failure\_by\_stage":\{\},"n\_failed\_flag\_true":0,"n\_failure\_events":0,"n\_parse\_failed":0\}}
\expandafter\def\csname ATRClaim@cross.gemma.full_method_budgeted@unit\endcsname{null}
\expandafter\def\csname ATRClaim@cross.gemma.full_method_budgeted@scope\endcsname{standalone\_full\_method}
\expandafter\def\csname ATRClaim@cross.gemma.full_method_budgeted@diagnostic@failure_as_zero\endcsname{true}
\expandafter\def\csname ATRClaim@cross.gemma.full_method_budgeted@diagnostic@failure_by_stage\endcsname{\{\}}
\expandafter\def\csname ATRClaim@cross.gemma.full_method_budgeted@diagnostic@n_failed_flag_true\endcsname{0}
\expandafter\def\csname ATRClaim@cross.gemma.full_method_budgeted@diagnostic@n_failure_events\endcsname{0}
\expandafter\def\csname ATRClaim@cross.gemma.full_method_budgeted@diagnostic@n_parse_failed\endcsname{0}
\expandafter\def\csname ATRClaim@cross.gemma.full_method_budgeted_minus_aggregate_only@value\endcsname{0.4841666666666667}
\expandafter\def\csname ATRClaim@cross.gemma.full_method_budgeted_minus_aggregate_only@n\endcsname{1200}
\expandafter\def\csname ATRClaim@cross.gemma.full_method_budgeted_minus_aggregate_only@windows\endcsname{174}
\expandafter\def\csname ATRClaim@cross.gemma.full_method_budgeted_minus_aggregate_only@ciLow\endcsname{0.45615514333895446}
\expandafter\def\csname ATRClaim@cross.gemma.full_method_budgeted_minus_aggregate_only@ciHigh\endcsname{0.5116279069767442}
\expandafter\def\csname ATRClaim@cross.gemma.full_method_budgeted_minus_aggregate_only@diagnostics\endcsname{\{\}}
\expandafter\def\csname ATRClaim@cross.gemma.full_method_budgeted_minus_aggregate_only@unit\endcsname{null}
\expandafter\def\csname ATRClaim@cross.gemma.full_method_budgeted_minus_aggregate_only@scope\endcsname{standalone\_full\_method}
\expandafter\def\csname ATRClaim@cross.gemma.full_method_budgeted_minus_exact@value\endcsname{-0.0075}
\expandafter\def\csname ATRClaim@cross.gemma.full_method_budgeted_minus_exact@n\endcsname{1200}
\expandafter\def\csname ATRClaim@cross.gemma.full_method_budgeted_minus_exact@windows\endcsname{174}
\expandafter\def\csname ATRClaim@cross.gemma.full_method_budgeted_minus_exact@ciLow\endcsname{-0.013300083125519535}
\expandafter\def\csname ATRClaim@cross.gemma.full_method_budgeted_minus_exact@ciHigh\endcsname{-0.002508361204013378}
\expandafter\def\csname ATRClaim@cross.gemma.full_method_budgeted_minus_exact@diagnostics\endcsname{\{\}}
\expandafter\def\csname ATRClaim@cross.gemma.full_method_budgeted_minus_exact@unit\endcsname{null}
\expandafter\def\csname ATRClaim@cross.gemma.full_method_budgeted_minus_exact@scope\endcsname{standalone\_full\_method}
\expandafter\def\csname ATRClaim@cross.gemma.full_method_budgeted_minus_policy_only@value\endcsname{0.6916666666666667}
\expandafter\def\csname ATRClaim@cross.gemma.full_method_budgeted_minus_policy_only@n\endcsname{1200}
\expandafter\def\csname ATRClaim@cross.gemma.full_method_budgeted_minus_policy_only@windows\endcsname{174}
\expandafter\def\csname ATRClaim@cross.gemma.full_method_budgeted_minus_policy_only@ciLow\endcsname{0.6606260296540363}
\expandafter\def\csname ATRClaim@cross.gemma.full_method_budgeted_minus_policy_only@ciHigh\endcsname{0.7215295095594347}
\expandafter\def\csname ATRClaim@cross.gemma.full_method_budgeted_minus_policy_only@diagnostics\endcsname{\{\}}
\expandafter\def\csname ATRClaim@cross.gemma.full_method_budgeted_minus_policy_only@unit\endcsname{null}
\expandafter\def\csname ATRClaim@cross.gemma.full_method_budgeted_minus_policy_only@scope\endcsname{standalone\_full\_method}
\expandafter\def\csname ATRClaim@cross.gemma.full_method_exact@value\endcsname{1.0}
\expandafter\def\csname ATRClaim@cross.gemma.full_method_exact@n\endcsname{1200}
\expandafter\def\csname ATRClaim@cross.gemma.full_method_exact@windows\endcsname{174}
\expandafter\def\csname ATRClaim@cross.gemma.full_method_exact@ciLow\endcsname{1.0}
\expandafter\def\csname ATRClaim@cross.gemma.full_method_exact@ciHigh\endcsname{1.0}
\expandafter\def\csname ATRClaim@cross.gemma.full_method_exact@diagnostics\endcsname{\{"failure\_as\_zero":true,"failure\_by\_stage":\{\},"n\_failed\_flag\_true":0,"n\_failure\_events":0,"n\_parse\_failed":0\}}
\expandafter\def\csname ATRClaim@cross.gemma.full_method_exact@unit\endcsname{null}
\expandafter\def\csname ATRClaim@cross.gemma.full_method_exact@scope\endcsname{standalone\_full\_method}
\expandafter\def\csname ATRClaim@cross.gemma.full_method_exact@diagnostic@failure_as_zero\endcsname{true}
\expandafter\def\csname ATRClaim@cross.gemma.full_method_exact@diagnostic@failure_by_stage\endcsname{\{\}}
\expandafter\def\csname ATRClaim@cross.gemma.full_method_exact@diagnostic@n_failed_flag_true\endcsname{0}
\expandafter\def\csname ATRClaim@cross.gemma.full_method_exact@diagnostic@n_failure_events\endcsname{0}
\expandafter\def\csname ATRClaim@cross.gemma.full_method_exact@diagnostic@n_parse_failed\endcsname{0}
\expandafter\def\csname ATRClaim@cross.gemma.full_method_policy_only@value\endcsname{0.30083333333333334}
\expandafter\def\csname ATRClaim@cross.gemma.full_method_policy_only@n\endcsname{1200}
\expandafter\def\csname ATRClaim@cross.gemma.full_method_policy_only@windows\endcsname{174}
\expandafter\def\csname ATRClaim@cross.gemma.full_method_policy_only@ciLow\endcsname{0.27056827820186596}
\expandafter\def\csname ATRClaim@cross.gemma.full_method_policy_only@ciHigh\endcsname{0.33249791144527985}
\expandafter\def\csname ATRClaim@cross.gemma.full_method_policy_only@diagnostics\endcsname{\{"failure\_as\_zero":true,"failure\_by\_stage":\{"parsing":488\},"n\_failed\_flag\_true":0,"n\_failure\_events":488,"n\_parse\_failed":488\}}
\expandafter\def\csname ATRClaim@cross.gemma.full_method_policy_only@unit\endcsname{null}
\expandafter\def\csname ATRClaim@cross.gemma.full_method_policy_only@scope\endcsname{standalone\_full\_method}
\expandafter\def\csname ATRClaim@cross.gemma.full_method_policy_only@diagnostic@failure_as_zero\endcsname{true}
\expandafter\def\csname ATRClaim@cross.gemma.full_method_policy_only@diagnostic@failure_by_stage\endcsname{\{"parsing":488\}}
\expandafter\def\csname ATRClaim@cross.gemma.full_method_policy_only@diagnostic@n_failed_flag_true\endcsname{0}
\expandafter\def\csname ATRClaim@cross.gemma.full_method_policy_only@diagnostic@n_failure_events\endcsname{488}
\expandafter\def\csname ATRClaim@cross.gemma.full_method_policy_only@diagnostic@n_parse_failed\endcsname{488}
\expandafter\def\csname ATRClaim@cross.gemma.guided_context@value\endcsname{0.43103448275862066}
\expandafter\def\csname ATRClaim@cross.gemma.guided_context@n\endcsname{174}
\expandafter\def\csname ATRClaim@cross.gemma.guided_context@windows\endcsname{174}
\expandafter\def\csname ATRClaim@cross.gemma.guided_context@ciLow\endcsname{0.3563218390804598}
\expandafter\def\csname ATRClaim@cross.gemma.guided_context@ciHigh\endcsname{0.5057471264367817}
\expandafter\def\csname ATRClaim@cross.gemma.guided_context@diagnostics\endcsname{\{"failure\_as\_zero":true,"failure\_by\_stage":\{"context\_packing":80\},"n\_failed\_flag\_true":80,"n\_failure\_events":80,"n\_parse\_failed":0\}}
\expandafter\def\csname ATRClaim@cross.gemma.guided_context@unit\endcsname{null}
\expandafter\def\csname ATRClaim@cross.gemma.guided_context@scope\endcsname{null}
\expandafter\def\csname ATRClaim@cross.gemma.guided_context@diagnostic@failure_as_zero\endcsname{true}
\expandafter\def\csname ATRClaim@cross.gemma.guided_context@diagnostic@failure_by_stage\endcsname{\{"context\_packing":80\}}
\expandafter\def\csname ATRClaim@cross.gemma.guided_context@diagnostic@n_failed_flag_true\endcsname{80}
\expandafter\def\csname ATRClaim@cross.gemma.guided_context@diagnostic@n_failure_events\endcsname{80}
\expandafter\def\csname ATRClaim@cross.gemma.guided_context@diagnostic@n_parse_failed\endcsname{0}
\expandafter\def\csname ATRClaim@cross.gemma.joint_common_budgeted@value\endcsname{1.0}
\expandafter\def\csname ATRClaim@cross.gemma.joint_common_budgeted@n\endcsname{94}
\expandafter\def\csname ATRClaim@cross.gemma.joint_common_budgeted@windows\endcsname{94}
\expandafter\def\csname ATRClaim@cross.gemma.joint_common_budgeted@ciLow\endcsname{1.0}
\expandafter\def\csname ATRClaim@cross.gemma.joint_common_budgeted@ciHigh\endcsname{1.0}
\expandafter\def\csname ATRClaim@cross.gemma.joint_common_budgeted@diagnostics\endcsname{\{"failure\_stage":\{"none":94\},"n\_context\_ineligible":0,"n\_parse\_failed":0,"n\_transport\_or\_harness\_failed":0,"n\_truncated\_at\_cap":0\}}
\expandafter\def\csname ATRClaim@cross.gemma.joint_common_budgeted@unit\endcsname{null}
\expandafter\def\csname ATRClaim@cross.gemma.joint_common_budgeted@scope\endcsname{joint\_common\_context\_eligible\_sensitivity}
\expandafter\def\csname ATRClaim@cross.gemma.joint_common_budgeted@diagnostic@failure_stage\endcsname{\{"none":94\}}
\expandafter\def\csname ATRClaim@cross.gemma.joint_common_budgeted@diagnostic@n_context_ineligible\endcsname{0}
\expandafter\def\csname ATRClaim@cross.gemma.joint_common_budgeted@diagnostic@n_parse_failed\endcsname{0}
\expandafter\def\csname ATRClaim@cross.gemma.joint_common_budgeted@diagnostic@n_transport_or_harness_failed\endcsname{0}
\expandafter\def\csname ATRClaim@cross.gemma.joint_common_budgeted@diagnostic@n_truncated_at_cap\endcsname{0}
\expandafter\def\csname ATRClaim@cross.gemma.joint_common_budgeted_minus_guided@value\endcsname{0.20212765957446807}
\expandafter\def\csname ATRClaim@cross.gemma.joint_common_budgeted_minus_guided@n\endcsname{94}
\expandafter\def\csname ATRClaim@cross.gemma.joint_common_budgeted_minus_guided@windows\endcsname{94}
\expandafter\def\csname ATRClaim@cross.gemma.joint_common_budgeted_minus_guided@ciLow\endcsname{0.1276595744680851}
\expandafter\def\csname ATRClaim@cross.gemma.joint_common_budgeted_minus_guided@ciHigh\endcsname{0.2872340425531915}
\expandafter\def\csname ATRClaim@cross.gemma.joint_common_budgeted_minus_guided@diagnostics\endcsname{\{\}}
\expandafter\def\csname ATRClaim@cross.gemma.joint_common_budgeted_minus_guided@unit\endcsname{null}
\expandafter\def\csname ATRClaim@cross.gemma.joint_common_budgeted_minus_guided@scope\endcsname{joint\_common\_context\_eligible\_sensitivity}
\expandafter\def\csname ATRClaim@cross.gemma.joint_common_guided@value\endcsname{0.7978723404255319}
\expandafter\def\csname ATRClaim@cross.gemma.joint_common_guided@n\endcsname{94}
\expandafter\def\csname ATRClaim@cross.gemma.joint_common_guided@windows\endcsname{94}
\expandafter\def\csname ATRClaim@cross.gemma.joint_common_guided@ciLow\endcsname{0.7127659574468085}
\expandafter\def\csname ATRClaim@cross.gemma.joint_common_guided@ciHigh\endcsname{0.8723404255319149}
\expandafter\def\csname ATRClaim@cross.gemma.joint_common_guided@diagnostics\endcsname{\{"failure\_stage":\{"none":94\},"n\_context\_ineligible":0,"n\_parse\_failed":0,"n\_transport\_or\_harness\_failed":0,"n\_truncated\_at\_cap":0\}}
\expandafter\def\csname ATRClaim@cross.gemma.joint_common_guided@unit\endcsname{null}
\expandafter\def\csname ATRClaim@cross.gemma.joint_common_guided@scope\endcsname{joint\_common\_context\_eligible\_sensitivity}
\expandafter\def\csname ATRClaim@cross.gemma.joint_common_guided@diagnostic@failure_stage\endcsname{\{"none":94\}}
\expandafter\def\csname ATRClaim@cross.gemma.joint_common_guided@diagnostic@n_context_ineligible\endcsname{0}
\expandafter\def\csname ATRClaim@cross.gemma.joint_common_guided@diagnostic@n_parse_failed\endcsname{0}
\expandafter\def\csname ATRClaim@cross.gemma.joint_common_guided@diagnostic@n_transport_or_harness_failed\endcsname{0}
\expandafter\def\csname ATRClaim@cross.gemma.joint_common_guided@diagnostic@n_truncated_at_cap\endcsname{0}
\expandafter\def\csname ATRClaim@cross.gemma.matched_budgeted@value\endcsname{0.9942528735632183}
\expandafter\def\csname ATRClaim@cross.gemma.matched_budgeted@n\endcsname{174}
\expandafter\def\csname ATRClaim@cross.gemma.matched_budgeted@windows\endcsname{174}
\expandafter\def\csname ATRClaim@cross.gemma.matched_budgeted@ciLow\endcsname{0.9827586206896551}
\expandafter\def\csname ATRClaim@cross.gemma.matched_budgeted@ciHigh\endcsname{1.0}
\expandafter\def\csname ATRClaim@cross.gemma.matched_budgeted@diagnostics\endcsname{\{"failure\_stage":\{"none":174\},"n\_context\_ineligible":0,"n\_parse\_failed":0,"n\_transport\_or\_harness\_failed":0,"n\_truncated\_at\_cap":0\}}
\expandafter\def\csname ATRClaim@cross.gemma.matched_budgeted@unit\endcsname{null}
\expandafter\def\csname ATRClaim@cross.gemma.matched_budgeted@scope\endcsname{null}
\expandafter\def\csname ATRClaim@cross.gemma.matched_budgeted@diagnostic@failure_stage\endcsname{\{"none":174\}}
\expandafter\def\csname ATRClaim@cross.gemma.matched_budgeted@diagnostic@n_context_ineligible\endcsname{0}
\expandafter\def\csname ATRClaim@cross.gemma.matched_budgeted@diagnostic@n_parse_failed\endcsname{0}
\expandafter\def\csname ATRClaim@cross.gemma.matched_budgeted@diagnostic@n_transport_or_harness_failed\endcsname{0}
\expandafter\def\csname ATRClaim@cross.gemma.matched_budgeted@diagnostic@n_truncated_at_cap\endcsname{0}
\expandafter\def\csname ATRClaim@cross.qwen.budgeted_minus_cot@value\endcsname{0.6091954022988506}
\expandafter\def\csname ATRClaim@cross.qwen.budgeted_minus_cot@n\endcsname{174}
\expandafter\def\csname ATRClaim@cross.qwen.budgeted_minus_cot@windows\endcsname{174}
\expandafter\def\csname ATRClaim@cross.qwen.budgeted_minus_cot@ciLow\endcsname{0.5402298850574713}
\expandafter\def\csname ATRClaim@cross.qwen.budgeted_minus_cot@ciHigh\endcsname{0.6839080459770115}
\expandafter\def\csname ATRClaim@cross.qwen.budgeted_minus_cot@diagnostics\endcsname{\{\}}
\expandafter\def\csname ATRClaim@cross.qwen.budgeted_minus_cot@unit\endcsname{null}
\expandafter\def\csname ATRClaim@cross.qwen.budgeted_minus_cot@scope\endcsname{null}
\expandafter\def\csname ATRClaim@cross.qwen.budgeted_minus_direct@value\endcsname{0.632183908045977}
\expandafter\def\csname ATRClaim@cross.qwen.budgeted_minus_direct@n\endcsname{174}
\expandafter\def\csname ATRClaim@cross.qwen.budgeted_minus_direct@windows\endcsname{174}
\expandafter\def\csname ATRClaim@cross.qwen.budgeted_minus_direct@ciLow\endcsname{0.5632183908045977}
\expandafter\def\csname ATRClaim@cross.qwen.budgeted_minus_direct@ciHigh\endcsname{0.7011494252873564}
\expandafter\def\csname ATRClaim@cross.qwen.budgeted_minus_direct@diagnostics\endcsname{\{\}}
\expandafter\def\csname ATRClaim@cross.qwen.budgeted_minus_direct@unit\endcsname{null}
\expandafter\def\csname ATRClaim@cross.qwen.budgeted_minus_direct@scope\endcsname{null}
\expandafter\def\csname ATRClaim@cross.qwen.budgeted_minus_guided@value\endcsname{0.6436781609195402}
\expandafter\def\csname ATRClaim@cross.qwen.budgeted_minus_guided@n\endcsname{174}
\expandafter\def\csname ATRClaim@cross.qwen.budgeted_minus_guided@windows\endcsname{174}
\expandafter\def\csname ATRClaim@cross.qwen.budgeted_minus_guided@ciLow\endcsname{0.5747126436781609}
\expandafter\def\csname ATRClaim@cross.qwen.budgeted_minus_guided@ciHigh\endcsname{0.7126436781609196}
\expandafter\def\csname ATRClaim@cross.qwen.budgeted_minus_guided@diagnostics\endcsname{\{\}}
\expandafter\def\csname ATRClaim@cross.qwen.budgeted_minus_guided@unit\endcsname{null}
\expandafter\def\csname ATRClaim@cross.qwen.budgeted_minus_guided@scope\endcsname{null}
\expandafter\def\csname ATRClaim@cross.qwen.commoneligible_budgeted@value\endcsname{0.9680851063829787}
\expandafter\def\csname ATRClaim@cross.qwen.commoneligible_budgeted@n\endcsname{94}
\expandafter\def\csname ATRClaim@cross.qwen.commoneligible_budgeted@windows\endcsname{94}
\expandafter\def\csname ATRClaim@cross.qwen.commoneligible_budgeted@ciLow\endcsname{0.925531914893617}
\expandafter\def\csname ATRClaim@cross.qwen.commoneligible_budgeted@ciHigh\endcsname{1.0}
\expandafter\def\csname ATRClaim@cross.qwen.commoneligible_budgeted@diagnostics\endcsname{\{"failure\_stage":\{"none":94\},"n\_context\_ineligible":0,"n\_parse\_failed":0,"n\_transport\_or\_harness\_failed":0,"n\_truncated\_at\_cap":0\}}
\expandafter\def\csname ATRClaim@cross.qwen.commoneligible_budgeted@unit\endcsname{null}
\expandafter\def\csname ATRClaim@cross.qwen.commoneligible_budgeted@scope\endcsname{common\_context\_eligible\_sensitivity}
\expandafter\def\csname ATRClaim@cross.qwen.commoneligible_budgeted@diagnostic@failure_stage\endcsname{\{"none":94\}}
\expandafter\def\csname ATRClaim@cross.qwen.commoneligible_budgeted@diagnostic@n_context_ineligible\endcsname{0}
\expandafter\def\csname ATRClaim@cross.qwen.commoneligible_budgeted@diagnostic@n_parse_failed\endcsname{0}
\expandafter\def\csname ATRClaim@cross.qwen.commoneligible_budgeted@diagnostic@n_transport_or_harness_failed\endcsname{0}
\expandafter\def\csname ATRClaim@cross.qwen.commoneligible_budgeted@diagnostic@n_truncated_at_cap\endcsname{0}
\expandafter\def\csname ATRClaim@cross.qwen.commoneligible_budgeted_minus_guided@value\endcsname{0.4574468085106383}
\expandafter\def\csname ATRClaim@cross.qwen.commoneligible_budgeted_minus_guided@n\endcsname{94}
\expandafter\def\csname ATRClaim@cross.qwen.commoneligible_budgeted_minus_guided@windows\endcsname{94}
\expandafter\def\csname ATRClaim@cross.qwen.commoneligible_budgeted_minus_guided@ciLow\endcsname{0.3617021276595745}
\expandafter\def\csname ATRClaim@cross.qwen.commoneligible_budgeted_minus_guided@ciHigh\endcsname{0.5531914893617021}
\expandafter\def\csname ATRClaim@cross.qwen.commoneligible_budgeted_minus_guided@diagnostics\endcsname{\{\}}
\expandafter\def\csname ATRClaim@cross.qwen.commoneligible_budgeted_minus_guided@unit\endcsname{null}
\expandafter\def\csname ATRClaim@cross.qwen.commoneligible_budgeted_minus_guided@scope\endcsname{common\_context\_eligible\_sensitivity}
\expandafter\def\csname ATRClaim@cross.qwen.commoneligible_guided@value\endcsname{0.5106382978723404}
\expandafter\def\csname ATRClaim@cross.qwen.commoneligible_guided@n\endcsname{94}
\expandafter\def\csname ATRClaim@cross.qwen.commoneligible_guided@windows\endcsname{94}
\expandafter\def\csname ATRClaim@cross.qwen.commoneligible_guided@ciLow\endcsname{0.4148936170212766}
\expandafter\def\csname ATRClaim@cross.qwen.commoneligible_guided@ciHigh\endcsname{0.6170212765957447}
\expandafter\def\csname ATRClaim@cross.qwen.commoneligible_guided@diagnostics\endcsname{\{"failure\_stage":\{"none":94\},"n\_context\_ineligible":0,"n\_parse\_failed":0,"n\_transport\_or\_harness\_failed":0,"n\_truncated\_at\_cap":0\}}
\expandafter\def\csname ATRClaim@cross.qwen.commoneligible_guided@unit\endcsname{null}
\expandafter\def\csname ATRClaim@cross.qwen.commoneligible_guided@scope\endcsname{common\_context\_eligible\_sensitivity}
\expandafter\def\csname ATRClaim@cross.qwen.commoneligible_guided@diagnostic@failure_stage\endcsname{\{"none":94\}}
\expandafter\def\csname ATRClaim@cross.qwen.commoneligible_guided@diagnostic@n_context_ineligible\endcsname{0}
\expandafter\def\csname ATRClaim@cross.qwen.commoneligible_guided@diagnostic@n_parse_failed\endcsname{0}
\expandafter\def\csname ATRClaim@cross.qwen.commoneligible_guided@diagnostic@n_transport_or_harness_failed\endcsname{0}
\expandafter\def\csname ATRClaim@cross.qwen.commoneligible_guided@diagnostic@n_truncated_at_cap\endcsname{0}
\expandafter\def\csname ATRClaim@cross.qwen.cot_context@value\endcsname{0.3103448275862069}
\expandafter\def\csname ATRClaim@cross.qwen.cot_context@n\endcsname{174}
\expandafter\def\csname ATRClaim@cross.qwen.cot_context@windows\endcsname{174}
\expandafter\def\csname ATRClaim@cross.qwen.cot_context@ciLow\endcsname{0.2413793103448276}
\expandafter\def\csname ATRClaim@cross.qwen.cot_context@ciHigh\endcsname{0.3793103448275862}
\expandafter\def\csname ATRClaim@cross.qwen.cot_context@diagnostics\endcsname{\{"failure\_as\_zero":true,"failure\_by\_stage":\{"context\_packing":80,"parsing":34\},"n\_failed\_flag\_true":80,"n\_failure\_events":114,"n\_parse\_failed":34\}}
\expandafter\def\csname ATRClaim@cross.qwen.cot_context@unit\endcsname{null}
\expandafter\def\csname ATRClaim@cross.qwen.cot_context@scope\endcsname{null}
\expandafter\def\csname ATRClaim@cross.qwen.cot_context@diagnostic@failure_as_zero\endcsname{true}
\expandafter\def\csname ATRClaim@cross.qwen.cot_context@diagnostic@failure_by_stage\endcsname{\{"context\_packing":80,"parsing":34\}}
\expandafter\def\csname ATRClaim@cross.qwen.cot_context@diagnostic@n_failed_flag_true\endcsname{80}
\expandafter\def\csname ATRClaim@cross.qwen.cot_context@diagnostic@n_failure_events\endcsname{114}
\expandafter\def\csname ATRClaim@cross.qwen.cot_context@diagnostic@n_parse_failed\endcsname{34}
\expandafter\def\csname ATRClaim@cross.qwen.direct_context@value\endcsname{0.28735632183908044}
\expandafter\def\csname ATRClaim@cross.qwen.direct_context@n\endcsname{174}
\expandafter\def\csname ATRClaim@cross.qwen.direct_context@windows\endcsname{174}
\expandafter\def\csname ATRClaim@cross.qwen.direct_context@ciLow\endcsname{0.22413793103448276}
\expandafter\def\csname ATRClaim@cross.qwen.direct_context@ciHigh\endcsname{0.3563218390804598}
\expandafter\def\csname ATRClaim@cross.qwen.direct_context@diagnostics\endcsname{\{"failure\_as\_zero":true,"failure\_by\_stage":\{"context\_packing":80,"parsing":13\},"n\_failed\_flag\_true":80,"n\_failure\_events":93,"n\_parse\_failed":13\}}
\expandafter\def\csname ATRClaim@cross.qwen.direct_context@unit\endcsname{null}
\expandafter\def\csname ATRClaim@cross.qwen.direct_context@scope\endcsname{null}
\expandafter\def\csname ATRClaim@cross.qwen.direct_context@diagnostic@failure_as_zero\endcsname{true}
\expandafter\def\csname ATRClaim@cross.qwen.direct_context@diagnostic@failure_by_stage\endcsname{\{"context\_packing":80,"parsing":13\}}
\expandafter\def\csname ATRClaim@cross.qwen.direct_context@diagnostic@n_failed_flag_true\endcsname{80}
\expandafter\def\csname ATRClaim@cross.qwen.direct_context@diagnostic@n_failure_events\endcsname{93}
\expandafter\def\csname ATRClaim@cross.qwen.direct_context@diagnostic@n_parse_failed\endcsname{13}
\expandafter\def\csname ATRClaim@cross.qwen.full_method_aggregate_only@value\endcsname{0.5158333333333334}
\expandafter\def\csname ATRClaim@cross.qwen.full_method_aggregate_only@n\endcsname{1200}
\expandafter\def\csname ATRClaim@cross.qwen.full_method_aggregate_only@windows\endcsname{174}
\expandafter\def\csname ATRClaim@cross.qwen.full_method_aggregate_only@ciLow\endcsname{0.48825503355704697}
\expandafter\def\csname ATRClaim@cross.qwen.full_method_aggregate_only@ciHigh\endcsname{0.5429752066115703}
\expandafter\def\csname ATRClaim@cross.qwen.full_method_aggregate_only@diagnostics\endcsname{\{"failure\_as\_zero":true,"failure\_by\_stage":\{\},"n\_failed\_flag\_true":0,"n\_failure\_events":0,"n\_parse\_failed":0\}}
\expandafter\def\csname ATRClaim@cross.qwen.full_method_aggregate_only@unit\endcsname{null}
\expandafter\def\csname ATRClaim@cross.qwen.full_method_aggregate_only@scope\endcsname{standalone\_full\_method}
\expandafter\def\csname ATRClaim@cross.qwen.full_method_aggregate_only@diagnostic@failure_as_zero\endcsname{true}
\expandafter\def\csname ATRClaim@cross.qwen.full_method_aggregate_only@diagnostic@failure_by_stage\endcsname{\{\}}
\expandafter\def\csname ATRClaim@cross.qwen.full_method_aggregate_only@diagnostic@n_failed_flag_true\endcsname{0}
\expandafter\def\csname ATRClaim@cross.qwen.full_method_aggregate_only@diagnostic@n_failure_events\endcsname{0}
\expandafter\def\csname ATRClaim@cross.qwen.full_method_aggregate_only@diagnostic@n_parse_failed\endcsname{0}
\expandafter\def\csname ATRClaim@cross.qwen.full_method_budgeted@value\endcsname{0.9108333333333334}
\expandafter\def\csname ATRClaim@cross.qwen.full_method_budgeted@n\endcsname{1200}
\expandafter\def\csname ATRClaim@cross.qwen.full_method_budgeted@windows\endcsname{174}
\expandafter\def\csname ATRClaim@cross.qwen.full_method_budgeted@ciLow\endcsname{0.8932443703085905}
\expandafter\def\csname ATRClaim@cross.qwen.full_method_budgeted@ciHigh\endcsname{0.9280936454849499}
\expandafter\def\csname ATRClaim@cross.qwen.full_method_budgeted@diagnostics\endcsname{\{"failure\_as\_zero":true,"failure\_by\_stage":\{\},"n\_failed\_flag\_true":0,"n\_failure\_events":0,"n\_parse\_failed":0\}}
\expandafter\def\csname ATRClaim@cross.qwen.full_method_budgeted@unit\endcsname{null}
\expandafter\def\csname ATRClaim@cross.qwen.full_method_budgeted@scope\endcsname{standalone\_full\_method}
\expandafter\def\csname ATRClaim@cross.qwen.full_method_budgeted@diagnostic@failure_as_zero\endcsname{true}
\expandafter\def\csname ATRClaim@cross.qwen.full_method_budgeted@diagnostic@failure_by_stage\endcsname{\{\}}
\expandafter\def\csname ATRClaim@cross.qwen.full_method_budgeted@diagnostic@n_failed_flag_true\endcsname{0}
\expandafter\def\csname ATRClaim@cross.qwen.full_method_budgeted@diagnostic@n_failure_events\endcsname{0}
\expandafter\def\csname ATRClaim@cross.qwen.full_method_budgeted@diagnostic@n_parse_failed\endcsname{0}
\expandafter\def\csname ATRClaim@cross.qwen.full_method_budgeted_minus_aggregate_only@value\endcsname{0.395}
\expandafter\def\csname ATRClaim@cross.qwen.full_method_budgeted_minus_aggregate_only@n\endcsname{1200}
\expandafter\def\csname ATRClaim@cross.qwen.full_method_budgeted_minus_aggregate_only@windows\endcsname{174}
\expandafter\def\csname ATRClaim@cross.qwen.full_method_budgeted_minus_aggregate_only@ciLow\endcsname{0.3594936708860759}
\expandafter\def\csname ATRClaim@cross.qwen.full_method_budgeted_minus_aggregate_only@ciHigh\endcsname{0.43013468013468015}
\expandafter\def\csname ATRClaim@cross.qwen.full_method_budgeted_minus_aggregate_only@diagnostics\endcsname{\{\}}
\expandafter\def\csname ATRClaim@cross.qwen.full_method_budgeted_minus_aggregate_only@unit\endcsname{null}
\expandafter\def\csname ATRClaim@cross.qwen.full_method_budgeted_minus_aggregate_only@scope\endcsname{standalone\_full\_method}
\expandafter\def\csname ATRClaim@cross.qwen.full_method_budgeted_minus_exact@value\endcsname{-0.030833333333333334}
\expandafter\def\csname ATRClaim@cross.qwen.full_method_budgeted_minus_exact@n\endcsname{1200}
\expandafter\def\csname ATRClaim@cross.qwen.full_method_budgeted_minus_exact@windows\endcsname{174}
\expandafter\def\csname ATRClaim@cross.qwen.full_method_budgeted_minus_exact@ciLow\endcsname{-0.04492512479201331}
\expandafter\def\csname ATRClaim@cross.qwen.full_method_budgeted_minus_exact@ciHigh\endcsname{-0.01728395061728395}
\expandafter\def\csname ATRClaim@cross.qwen.full_method_budgeted_minus_exact@diagnostics\endcsname{\{\}}
\expandafter\def\csname ATRClaim@cross.qwen.full_method_budgeted_minus_exact@unit\endcsname{null}
\expandafter\def\csname ATRClaim@cross.qwen.full_method_budgeted_minus_exact@scope\endcsname{standalone\_full\_method}
\expandafter\def\csname ATRClaim@cross.qwen.full_method_budgeted_minus_policy_only@value\endcsname{0.3858333333333333}
\expandafter\def\csname ATRClaim@cross.qwen.full_method_budgeted_minus_policy_only@n\endcsname{1200}
\expandafter\def\csname ATRClaim@cross.qwen.full_method_budgeted_minus_policy_only@windows\endcsname{174}
\expandafter\def\csname ATRClaim@cross.qwen.full_method_budgeted_minus_policy_only@ciLow\endcsname{0.3502495840266223}
\expandafter\def\csname ATRClaim@cross.qwen.full_method_budgeted_minus_policy_only@ciHigh\endcsname{0.4211409395973154}
\expandafter\def\csname ATRClaim@cross.qwen.full_method_budgeted_minus_policy_only@diagnostics\endcsname{\{\}}
\expandafter\def\csname ATRClaim@cross.qwen.full_method_budgeted_minus_policy_only@unit\endcsname{null}
\expandafter\def\csname ATRClaim@cross.qwen.full_method_budgeted_minus_policy_only@scope\endcsname{standalone\_full\_method}
\expandafter\def\csname ATRClaim@cross.qwen.full_method_exact@value\endcsname{0.9416666666666667}
\expandafter\def\csname ATRClaim@cross.qwen.full_method_exact@n\endcsname{1200}
\expandafter\def\csname ATRClaim@cross.qwen.full_method_exact@windows\endcsname{174}
\expandafter\def\csname ATRClaim@cross.qwen.full_method_exact@ciLow\endcsname{0.9265822784810127}
\expandafter\def\csname ATRClaim@cross.qwen.full_method_exact@ciHigh\endcsname{0.9556113902847572}
\expandafter\def\csname ATRClaim@cross.qwen.full_method_exact@diagnostics\endcsname{\{"failure\_as\_zero":true,"failure\_by\_stage":\{\},"n\_failed\_flag\_true":0,"n\_failure\_events":0,"n\_parse\_failed":0\}}
\expandafter\def\csname ATRClaim@cross.qwen.full_method_exact@unit\endcsname{null}
\expandafter\def\csname ATRClaim@cross.qwen.full_method_exact@scope\endcsname{standalone\_full\_method}
\expandafter\def\csname ATRClaim@cross.qwen.full_method_exact@diagnostic@failure_as_zero\endcsname{true}
\expandafter\def\csname ATRClaim@cross.qwen.full_method_exact@diagnostic@failure_by_stage\endcsname{\{\}}
\expandafter\def\csname ATRClaim@cross.qwen.full_method_exact@diagnostic@n_failed_flag_true\endcsname{0}
\expandafter\def\csname ATRClaim@cross.qwen.full_method_exact@diagnostic@n_failure_events\endcsname{0}
\expandafter\def\csname ATRClaim@cross.qwen.full_method_exact@diagnostic@n_parse_failed\endcsname{0}
\expandafter\def\csname ATRClaim@cross.qwen.full_method_policy_only@value\endcsname{0.525}
\expandafter\def\csname ATRClaim@cross.qwen.full_method_policy_only@n\endcsname{1200}
\expandafter\def\csname ATRClaim@cross.qwen.full_method_policy_only@windows\endcsname{174}
\expandafter\def\csname ATRClaim@cross.qwen.full_method_policy_only@ciLow\endcsname{0.4983221476510067}
\expandafter\def\csname ATRClaim@cross.qwen.full_method_policy_only@ciHigh\endcsname{0.5516372795969773}
\expandafter\def\csname ATRClaim@cross.qwen.full_method_policy_only@diagnostics\endcsname{\{"failure\_as\_zero":true,"failure\_by\_stage":\{\},"n\_failed\_flag\_true":0,"n\_failure\_events":0,"n\_parse\_failed":0\}}
\expandafter\def\csname ATRClaim@cross.qwen.full_method_policy_only@unit\endcsname{null}
\expandafter\def\csname ATRClaim@cross.qwen.full_method_policy_only@scope\endcsname{standalone\_full\_method}
\expandafter\def\csname ATRClaim@cross.qwen.full_method_policy_only@diagnostic@failure_as_zero\endcsname{true}
\expandafter\def\csname ATRClaim@cross.qwen.full_method_policy_only@diagnostic@failure_by_stage\endcsname{\{\}}
\expandafter\def\csname ATRClaim@cross.qwen.full_method_policy_only@diagnostic@n_failed_flag_true\endcsname{0}
\expandafter\def\csname ATRClaim@cross.qwen.full_method_policy_only@diagnostic@n_failure_events\endcsname{0}
\expandafter\def\csname ATRClaim@cross.qwen.full_method_policy_only@diagnostic@n_parse_failed\endcsname{0}
\expandafter\def\csname ATRClaim@cross.qwen.guided_context@value\endcsname{0.27586206896551724}
\expandafter\def\csname ATRClaim@cross.qwen.guided_context@n\endcsname{174}
\expandafter\def\csname ATRClaim@cross.qwen.guided_context@windows\endcsname{174}
\expandafter\def\csname ATRClaim@cross.qwen.guided_context@ciLow\endcsname{0.21264367816091953}
\expandafter\def\csname ATRClaim@cross.qwen.guided_context@ciHigh\endcsname{0.3448275862068966}
\expandafter\def\csname ATRClaim@cross.qwen.guided_context@diagnostics\endcsname{\{"failure\_as\_zero":true,"failure\_by\_stage":\{"context\_packing":80\},"n\_failed\_flag\_true":80,"n\_failure\_events":80,"n\_parse\_failed":0\}}
\expandafter\def\csname ATRClaim@cross.qwen.guided_context@unit\endcsname{null}
\expandafter\def\csname ATRClaim@cross.qwen.guided_context@scope\endcsname{null}
\expandafter\def\csname ATRClaim@cross.qwen.guided_context@diagnostic@failure_as_zero\endcsname{true}
\expandafter\def\csname ATRClaim@cross.qwen.guided_context@diagnostic@failure_by_stage\endcsname{\{"context\_packing":80\}}
\expandafter\def\csname ATRClaim@cross.qwen.guided_context@diagnostic@n_failed_flag_true\endcsname{80}
\expandafter\def\csname ATRClaim@cross.qwen.guided_context@diagnostic@n_failure_events\endcsname{80}
\expandafter\def\csname ATRClaim@cross.qwen.guided_context@diagnostic@n_parse_failed\endcsname{0}
\expandafter\def\csname ATRClaim@cross.qwen.joint_common_budgeted@value\endcsname{0.9680851063829787}
\expandafter\def\csname ATRClaim@cross.qwen.joint_common_budgeted@n\endcsname{94}
\expandafter\def\csname ATRClaim@cross.qwen.joint_common_budgeted@windows\endcsname{94}
\expandafter\def\csname ATRClaim@cross.qwen.joint_common_budgeted@ciLow\endcsname{0.925531914893617}
\expandafter\def\csname ATRClaim@cross.qwen.joint_common_budgeted@ciHigh\endcsname{1.0}
\expandafter\def\csname ATRClaim@cross.qwen.joint_common_budgeted@diagnostics\endcsname{\{"failure\_stage":\{"none":94\},"n\_context\_ineligible":0,"n\_parse\_failed":0,"n\_transport\_or\_harness\_failed":0,"n\_truncated\_at\_cap":0\}}
\expandafter\def\csname ATRClaim@cross.qwen.joint_common_budgeted@unit\endcsname{null}
\expandafter\def\csname ATRClaim@cross.qwen.joint_common_budgeted@scope\endcsname{joint\_common\_context\_eligible\_sensitivity}
\expandafter\def\csname ATRClaim@cross.qwen.joint_common_budgeted@diagnostic@failure_stage\endcsname{\{"none":94\}}
\expandafter\def\csname ATRClaim@cross.qwen.joint_common_budgeted@diagnostic@n_context_ineligible\endcsname{0}
\expandafter\def\csname ATRClaim@cross.qwen.joint_common_budgeted@diagnostic@n_parse_failed\endcsname{0}
\expandafter\def\csname ATRClaim@cross.qwen.joint_common_budgeted@diagnostic@n_transport_or_harness_failed\endcsname{0}
\expandafter\def\csname ATRClaim@cross.qwen.joint_common_budgeted@diagnostic@n_truncated_at_cap\endcsname{0}
\expandafter\def\csname ATRClaim@cross.qwen.joint_common_budgeted_minus_guided@value\endcsname{0.4574468085106383}
\expandafter\def\csname ATRClaim@cross.qwen.joint_common_budgeted_minus_guided@n\endcsname{94}
\expandafter\def\csname ATRClaim@cross.qwen.joint_common_budgeted_minus_guided@windows\endcsname{94}
\expandafter\def\csname ATRClaim@cross.qwen.joint_common_budgeted_minus_guided@ciLow\endcsname{0.3617021276595745}
\expandafter\def\csname ATRClaim@cross.qwen.joint_common_budgeted_minus_guided@ciHigh\endcsname{0.5531914893617021}
\expandafter\def\csname ATRClaim@cross.qwen.joint_common_budgeted_minus_guided@diagnostics\endcsname{\{\}}
\expandafter\def\csname ATRClaim@cross.qwen.joint_common_budgeted_minus_guided@unit\endcsname{null}
\expandafter\def\csname ATRClaim@cross.qwen.joint_common_budgeted_minus_guided@scope\endcsname{joint\_common\_context\_eligible\_sensitivity}
\expandafter\def\csname ATRClaim@cross.qwen.joint_common_guided@value\endcsname{0.5106382978723404}
\expandafter\def\csname ATRClaim@cross.qwen.joint_common_guided@n\endcsname{94}
\expandafter\def\csname ATRClaim@cross.qwen.joint_common_guided@windows\endcsname{94}
\expandafter\def\csname ATRClaim@cross.qwen.joint_common_guided@ciLow\endcsname{0.4148936170212766}
\expandafter\def\csname ATRClaim@cross.qwen.joint_common_guided@ciHigh\endcsname{0.6170212765957447}
\expandafter\def\csname ATRClaim@cross.qwen.joint_common_guided@diagnostics\endcsname{\{"failure\_stage":\{"none":94\},"n\_context\_ineligible":0,"n\_parse\_failed":0,"n\_transport\_or\_harness\_failed":0,"n\_truncated\_at\_cap":0\}}
\expandafter\def\csname ATRClaim@cross.qwen.joint_common_guided@unit\endcsname{null}
\expandafter\def\csname ATRClaim@cross.qwen.joint_common_guided@scope\endcsname{joint\_common\_context\_eligible\_sensitivity}
\expandafter\def\csname ATRClaim@cross.qwen.joint_common_guided@diagnostic@failure_stage\endcsname{\{"none":94\}}
\expandafter\def\csname ATRClaim@cross.qwen.joint_common_guided@diagnostic@n_context_ineligible\endcsname{0}
\expandafter\def\csname ATRClaim@cross.qwen.joint_common_guided@diagnostic@n_parse_failed\endcsname{0}
\expandafter\def\csname ATRClaim@cross.qwen.joint_common_guided@diagnostic@n_transport_or_harness_failed\endcsname{0}
\expandafter\def\csname ATRClaim@cross.qwen.joint_common_guided@diagnostic@n_truncated_at_cap\endcsname{0}
\expandafter\def\csname ATRClaim@cross.qwen.matched_budgeted@value\endcsname{0.9195402298850575}
\expandafter\def\csname ATRClaim@cross.qwen.matched_budgeted@n\endcsname{174}
\expandafter\def\csname ATRClaim@cross.qwen.matched_budgeted@windows\endcsname{174}
\expandafter\def\csname ATRClaim@cross.qwen.matched_budgeted@ciLow\endcsname{0.8793103448275862}
\expandafter\def\csname ATRClaim@cross.qwen.matched_budgeted@ciHigh\endcsname{0.9597701149425287}
\expandafter\def\csname ATRClaim@cross.qwen.matched_budgeted@diagnostics\endcsname{\{"failure\_stage":\{"none":174\},"n\_context\_ineligible":0,"n\_parse\_failed":0,"n\_transport\_or\_harness\_failed":0,"n\_truncated\_at\_cap":0\}}
\expandafter\def\csname ATRClaim@cross.qwen.matched_budgeted@unit\endcsname{null}
\expandafter\def\csname ATRClaim@cross.qwen.matched_budgeted@scope\endcsname{null}
\expandafter\def\csname ATRClaim@cross.qwen.matched_budgeted@diagnostic@failure_stage\endcsname{\{"none":174\}}
\expandafter\def\csname ATRClaim@cross.qwen.matched_budgeted@diagnostic@n_context_ineligible\endcsname{0}
\expandafter\def\csname ATRClaim@cross.qwen.matched_budgeted@diagnostic@n_parse_failed\endcsname{0}
\expandafter\def\csname ATRClaim@cross.qwen.matched_budgeted@diagnostic@n_transport_or_harness_failed\endcsname{0}
\expandafter\def\csname ATRClaim@cross.qwen.matched_budgeted@diagnostic@n_truncated_at_cap\endcsname{0}
\expandafter\def\csname ATRClaim@package_d.distinct_rmse.m1024@value\endcsname{0.03331503413584331}
\expandafter\def\csname ATRClaim@package_d.distinct_rmse.m1024@n\endcsname{160}
\expandafter\def\csname ATRClaim@package_d.distinct_rmse.m1024@windows\endcsname{null}
\expandafter\def\csname ATRClaim@package_d.distinct_rmse.m1024@ciLow\endcsname{0.02983077603330961}
\expandafter\def\csname ATRClaim@package_d.distinct_rmse.m1024@ciHigh\endcsname{0.036493042678079166}
\expandafter\def\csname ATRClaim@package_d.distinct_rmse.m1024@diagnostics\endcsname{\{\}}
\expandafter\def\csname ATRClaim@package_d.distinct_rmse.m1024@unit\endcsname{null}
\expandafter\def\csname ATRClaim@package_d.distinct_rmse.m1024@scope\endcsname{null}
\expandafter\def\csname ATRClaim@package_d.distinct_rmse.m2048@value\endcsname{0.02342693270790632}
\expandafter\def\csname ATRClaim@package_d.distinct_rmse.m2048@n\endcsname{160}
\expandafter\def\csname ATRClaim@package_d.distinct_rmse.m2048@windows\endcsname{null}
\expandafter\def\csname ATRClaim@package_d.distinct_rmse.m2048@ciLow\endcsname{0.02088585394380089}
\expandafter\def\csname ATRClaim@package_d.distinct_rmse.m2048@ciHigh\endcsname{0.02595197631026463}
\expandafter\def\csname ATRClaim@package_d.distinct_rmse.m2048@diagnostics\endcsname{\{\}}
\expandafter\def\csname ATRClaim@package_d.distinct_rmse.m2048@unit\endcsname{null}
\expandafter\def\csname ATRClaim@package_d.distinct_rmse.m2048@scope\endcsname{null}
\expandafter\def\csname ATRClaim@package_d.distinct_rmse.m256@value\endcsname{0.06465024502395066}
\expandafter\def\csname ATRClaim@package_d.distinct_rmse.m256@n\endcsname{240}
\expandafter\def\csname ATRClaim@package_d.distinct_rmse.m256@windows\endcsname{null}
\expandafter\def\csname ATRClaim@package_d.distinct_rmse.m256@ciLow\endcsname{0.05915879716326972}
\expandafter\def\csname ATRClaim@package_d.distinct_rmse.m256@ciHigh\endcsname{0.07018523610597084}
\expandafter\def\csname ATRClaim@package_d.distinct_rmse.m256@diagnostics\endcsname{\{\}}
\expandafter\def\csname ATRClaim@package_d.distinct_rmse.m256@unit\endcsname{null}
\expandafter\def\csname ATRClaim@package_d.distinct_rmse.m256@scope\endcsname{null}
\expandafter\def\csname ATRClaim@package_d.distinct_rmse.m4096@value\endcsname{0.01720513206692743}
\expandafter\def\csname ATRClaim@package_d.distinct_rmse.m4096@n\endcsname{120}
\expandafter\def\csname ATRClaim@package_d.distinct_rmse.m4096@windows\endcsname{null}
\expandafter\def\csname ATRClaim@package_d.distinct_rmse.m4096@ciLow\endcsname{0.015166496887359264}
\expandafter\def\csname ATRClaim@package_d.distinct_rmse.m4096@ciHigh\endcsname{0.01891330859231902}
\expandafter\def\csname ATRClaim@package_d.distinct_rmse.m4096@diagnostics\endcsname{\{\}}
\expandafter\def\csname ATRClaim@package_d.distinct_rmse.m4096@unit\endcsname{null}
\expandafter\def\csname ATRClaim@package_d.distinct_rmse.m4096@scope\endcsname{null}
\expandafter\def\csname ATRClaim@package_d.distinct_rmse.m512@value\endcsname{0.046010530511407605}
\expandafter\def\csname ATRClaim@package_d.distinct_rmse.m512@n\endcsname{200}
\expandafter\def\csname ATRClaim@package_d.distinct_rmse.m512@windows\endcsname{null}
\expandafter\def\csname ATRClaim@package_d.distinct_rmse.m512@ciLow\endcsname{0.042361902576289294}
\expandafter\def\csname ATRClaim@package_d.distinct_rmse.m512@ciHigh\endcsname{0.0497927922371584}
\expandafter\def\csname ATRClaim@package_d.distinct_rmse.m512@diagnostics\endcsname{\{\}}
\expandafter\def\csname ATRClaim@package_d.distinct_rmse.m512@unit\endcsname{null}
\expandafter\def\csname ATRClaim@package_d.distinct_rmse.m512@scope\endcsname{null}
\expandafter\def\csname ATRClaim@package_d.distinct_rmse.m8192@value\endcsname{0.01096216952654259}
\expandafter\def\csname ATRClaim@package_d.distinct_rmse.m8192@n\endcsname{120}
\expandafter\def\csname ATRClaim@package_d.distinct_rmse.m8192@windows\endcsname{null}
\expandafter\def\csname ATRClaim@package_d.distinct_rmse.m8192@ciLow\endcsname{0.009474647412477014}
\expandafter\def\csname ATRClaim@package_d.distinct_rmse.m8192@ciHigh\endcsname{0.012367360013895145}
\expandafter\def\csname ATRClaim@package_d.distinct_rmse.m8192@diagnostics\endcsname{\{\}}
\expandafter\def\csname ATRClaim@package_d.distinct_rmse.m8192@unit\endcsname{null}
\expandafter\def\csname ATRClaim@package_d.distinct_rmse.m8192@scope\endcsname{null}
\expandafter\def\csname ATRClaim@package_d.m2048_utf8_crossover@value\endcsname{139.26288589691282}
\expandafter\def\csname ATRClaim@package_d.m2048_utf8_crossover@n\endcsname{20000}
\expandafter\def\csname ATRClaim@package_d.m2048_utf8_crossover@windows\endcsname{null}
\expandafter\def\csname ATRClaim@package_d.m2048_utf8_crossover@ciLow\endcsname{138.68620553051875}
\expandafter\def\csname ATRClaim@package_d.m2048_utf8_crossover@ciHigh\endcsname{139.86348286023554}
\expandafter\def\csname ATRClaim@package_d.m2048_utf8_crossover@diagnostics\endcsname{\{\}}
\expandafter\def\csname ATRClaim@package_d.m2048_utf8_crossover@unit\endcsname{null}
\expandafter\def\csname ATRClaim@package_d.m2048_utf8_crossover@scope\endcsname{null}
\expandafter\def\csname ATRClaim@scale.aggregate_only.accuracy@value\endcsname{0.5217939027462837}
\expandafter\def\csname ATRClaim@scale.aggregate_only.accuracy@n\endcsname{3969}
\expandafter\def\csname ATRClaim@scale.aggregate_only.accuracy@windows\endcsname{174}
\expandafter\def\csname ATRClaim@scale.aggregate_only.accuracy@ciLow\endcsname{0.5112669003505258}
\expandafter\def\csname ATRClaim@scale.aggregate_only.accuracy@ciHigh\endcsname{0.5319917440660474}
\expandafter\def\csname ATRClaim@scale.aggregate_only.accuracy@diagnostics\endcsname{\{"failure\_stage":\{"none":3969\},"n\_context\_ineligible":0,"n\_parse\_failed":0,"n\_transport\_or\_harness\_failed":0,"n\_truncated\_at\_cap":0\}}
\expandafter\def\csname ATRClaim@scale.aggregate_only.accuracy@unit\endcsname{null}
\expandafter\def\csname ATRClaim@scale.aggregate_only.accuracy@scope\endcsname{null}
\expandafter\def\csname ATRClaim@scale.aggregate_only.accuracy@diagnostic@failure_stage\endcsname{\{"none":3969\}}
\expandafter\def\csname ATRClaim@scale.aggregate_only.accuracy@diagnostic@n_context_ineligible\endcsname{0}
\expandafter\def\csname ATRClaim@scale.aggregate_only.accuracy@diagnostic@n_parse_failed\endcsname{0}
\expandafter\def\csname ATRClaim@scale.aggregate_only.accuracy@diagnostic@n_transport_or_harness_failed\endcsname{0}
\expandafter\def\csname ATRClaim@scale.aggregate_only.accuracy@diagnostic@n_truncated_at_cap\endcsname{0}
\expandafter\def\csname ATRClaim@scale.budgeted.accuracy@value\endcsname{0.889896699420509}
\expandafter\def\csname ATRClaim@scale.budgeted.accuracy@n\endcsname{3969}
\expandafter\def\csname ATRClaim@scale.budgeted.accuracy@windows\endcsname{174}
\expandafter\def\csname ATRClaim@scale.budgeted.accuracy@ciLow\endcsname{0.878735913767761}
\expandafter\def\csname ATRClaim@scale.budgeted.accuracy@ciHigh\endcsname{0.900851276915373}
\expandafter\def\csname ATRClaim@scale.budgeted.accuracy@diagnostics\endcsname{\{"failure\_stage":\{"none":3969\},"n\_context\_ineligible":0,"n\_parse\_failed":0,"n\_transport\_or\_harness\_failed":0,"n\_truncated\_at\_cap":0\}}
\expandafter\def\csname ATRClaim@scale.budgeted.accuracy@unit\endcsname{null}
\expandafter\def\csname ATRClaim@scale.budgeted.accuracy@scope\endcsname{null}
\expandafter\def\csname ATRClaim@scale.budgeted.accuracy@diagnostic@failure_stage\endcsname{\{"none":3969\}}
\expandafter\def\csname ATRClaim@scale.budgeted.accuracy@diagnostic@n_context_ineligible\endcsname{0}
\expandafter\def\csname ATRClaim@scale.budgeted.accuracy@diagnostic@n_parse_failed\endcsname{0}
\expandafter\def\csname ATRClaim@scale.budgeted.accuracy@diagnostic@n_transport_or_harness_failed\endcsname{0}
\expandafter\def\csname ATRClaim@scale.budgeted.accuracy@diagnostic@n_truncated_at_cap\endcsname{0}
\expandafter\def\csname ATRClaim@scale.budgeted_minus_aggregate_only@value\endcsname{0.36810279667422524}
\expandafter\def\csname ATRClaim@scale.budgeted_minus_aggregate_only@n\endcsname{3969}
\expandafter\def\csname ATRClaim@scale.budgeted_minus_aggregate_only@windows\endcsname{174}
\expandafter\def\csname ATRClaim@scale.budgeted_minus_aggregate_only@ciLow\endcsname{0.35139816877010643}
\expandafter\def\csname ATRClaim@scale.budgeted_minus_aggregate_only@ciHigh\endcsname{0.38481338481338484}
\expandafter\def\csname ATRClaim@scale.budgeted_minus_aggregate_only@diagnostics\endcsname{\{\}}
\expandafter\def\csname ATRClaim@scale.budgeted_minus_aggregate_only@unit\endcsname{null}
\expandafter\def\csname ATRClaim@scale.budgeted_minus_aggregate_only@scope\endcsname{null}
\expandafter\def\csname ATRClaim@scale.budgeted_minus_exact@value\endcsname{-0.040312421264802216}
\expandafter\def\csname ATRClaim@scale.budgeted_minus_exact@n\endcsname{3969}
\expandafter\def\csname ATRClaim@scale.budgeted_minus_exact@windows\endcsname{174}
\expandafter\def\csname ATRClaim@scale.budgeted_minus_exact@ciLow\endcsname{-0.04873599166015116}
\expandafter\def\csname ATRClaim@scale.budgeted_minus_exact@ciHigh\endcsname{-0.03237025147137507}
\expandafter\def\csname ATRClaim@scale.budgeted_minus_exact@diagnostics\endcsname{\{\}}
\expandafter\def\csname ATRClaim@scale.budgeted_minus_exact@unit\endcsname{null}
\expandafter\def\csname ATRClaim@scale.budgeted_minus_exact@scope\endcsname{null}
\expandafter\def\csname ATRClaim@scale.budgeted_minus_policy_only@value\endcsname{0.36407155454774504}
\expandafter\def\csname ATRClaim@scale.budgeted_minus_policy_only@n\endcsname{3969}
\expandafter\def\csname ATRClaim@scale.budgeted_minus_policy_only@windows\endcsname{174}
\expandafter\def\csname ATRClaim@scale.budgeted_minus_policy_only@ciLow\endcsname{0.3467399450412191}
\expandafter\def\csname ATRClaim@scale.budgeted_minus_policy_only@ciHigh\endcsname{0.3813041263372389}
\expandafter\def\csname ATRClaim@scale.budgeted_minus_policy_only@diagnostics\endcsname{\{\}}
\expandafter\def\csname ATRClaim@scale.budgeted_minus_policy_only@unit\endcsname{null}
\expandafter\def\csname ATRClaim@scale.budgeted_minus_policy_only@scope\endcsname{null}
\expandafter\def\csname ATRClaim@scale.counterfactual.accuracy@value\endcsname{0.9281934996220711}
\expandafter\def\csname ATRClaim@scale.counterfactual.accuracy@n\endcsname{3969}
\expandafter\def\csname ATRClaim@scale.counterfactual.accuracy@windows\endcsname{174}
\expandafter\def\csname ATRClaim@scale.counterfactual.accuracy@ciLow\endcsname{0.9194373401534527}
\expandafter\def\csname ATRClaim@scale.counterfactual.accuracy@ciHigh\endcsname{0.9367432666172474}
\expandafter\def\csname ATRClaim@scale.counterfactual.accuracy@diagnostics\endcsname{\{"failure\_stage":\{"none":3969\},"n\_context\_ineligible":0,"n\_parse\_failed":0,"n\_transport\_or\_harness\_failed":0,"n\_truncated\_at\_cap":0\}}
\expandafter\def\csname ATRClaim@scale.counterfactual.accuracy@unit\endcsname{null}
\expandafter\def\csname ATRClaim@scale.counterfactual.accuracy@scope\endcsname{appendix\_mechanism\_diagnostic}
\expandafter\def\csname ATRClaim@scale.counterfactual.accuracy@diagnostic@failure_stage\endcsname{\{"none":3969\}}
\expandafter\def\csname ATRClaim@scale.counterfactual.accuracy@diagnostic@n_context_ineligible\endcsname{0}
\expandafter\def\csname ATRClaim@scale.counterfactual.accuracy@diagnostic@n_parse_failed\endcsname{0}
\expandafter\def\csname ATRClaim@scale.counterfactual.accuracy@diagnostic@n_transport_or_harness_failed\endcsname{0}
\expandafter\def\csname ATRClaim@scale.counterfactual.accuracy@diagnostic@n_truncated_at_cap\endcsname{0}
\expandafter\def\csname ATRClaim@scale.exact.accuracy@value\endcsname{0.9302091206853111}
\expandafter\def\csname ATRClaim@scale.exact.accuracy@n\endcsname{3969}
\expandafter\def\csname ATRClaim@scale.exact.accuracy@windows\endcsname{174}
\expandafter\def\csname ATRClaim@scale.exact.accuracy@ciLow\endcsname{0.9222526787939198}
\expandafter\def\csname ATRClaim@scale.exact.accuracy@ciHigh\endcsname{0.9383059418457649}
\expandafter\def\csname ATRClaim@scale.exact.accuracy@diagnostics\endcsname{\{"failure\_stage":\{"none":3969\},"n\_context\_ineligible":0,"n\_parse\_failed":0,"n\_transport\_or\_harness\_failed":0,"n\_truncated\_at\_cap":0\}}
\expandafter\def\csname ATRClaim@scale.exact.accuracy@unit\endcsname{null}
\expandafter\def\csname ATRClaim@scale.exact.accuracy@scope\endcsname{null}
\expandafter\def\csname ATRClaim@scale.exact.accuracy@diagnostic@failure_stage\endcsname{\{"none":3969\}}
\expandafter\def\csname ATRClaim@scale.exact.accuracy@diagnostic@n_context_ineligible\endcsname{0}
\expandafter\def\csname ATRClaim@scale.exact.accuracy@diagnostic@n_parse_failed\endcsname{0}
\expandafter\def\csname ATRClaim@scale.exact.accuracy@diagnostic@n_transport_or_harness_failed\endcsname{0}
\expandafter\def\csname ATRClaim@scale.exact.accuracy@diagnostic@n_truncated_at_cap\endcsname{0}
\expandafter\def\csname ATRClaim@scale.policy_only.accuracy@value\endcsname{0.5258251448727639}
\expandafter\def\csname ATRClaim@scale.policy_only.accuracy@n\endcsname{3969}
\expandafter\def\csname ATRClaim@scale.policy_only.accuracy@windows\endcsname{174}
\expandafter\def\csname ATRClaim@scale.policy_only.accuracy@ciLow\endcsname{0.515869140625}
\expandafter\def\csname ATRClaim@scale.policy_only.accuracy@ciHigh\endcsname{0.5355907068709836}
\expandafter\def\csname ATRClaim@scale.policy_only.accuracy@diagnostics\endcsname{\{"failure\_stage":\{"none":3969\},"n\_context\_ineligible":0,"n\_parse\_failed":0,"n\_transport\_or\_harness\_failed":0,"n\_truncated\_at\_cap":0\}}
\expandafter\def\csname ATRClaim@scale.policy_only.accuracy@unit\endcsname{null}
\expandafter\def\csname ATRClaim@scale.policy_only.accuracy@scope\endcsname{null}
\expandafter\def\csname ATRClaim@scale.policy_only.accuracy@diagnostic@failure_stage\endcsname{\{"none":3969\}}
\expandafter\def\csname ATRClaim@scale.policy_only.accuracy@diagnostic@n_context_ineligible\endcsname{0}
\expandafter\def\csname ATRClaim@scale.policy_only.accuracy@diagnostic@n_parse_failed\endcsname{0}
\expandafter\def\csname ATRClaim@scale.policy_only.accuracy@diagnostic@n_transport_or_harness_failed\endcsname{0}
\expandafter\def\csname ATRClaim@scale.policy_only.accuracy@diagnostic@n_truncated_at_cap\endcsname{0}
\expandafter\def\csname ATRClaim@tool.test.failure_as_zero_accuracy@value\endcsname{0.7909}
\expandafter\def\csname ATRClaim@tool.test.failure_as_zero_accuracy@n\endcsname{110}
\expandafter\def\csname ATRClaim@tool.test.failure_as_zero_accuracy@windows\endcsname{null}
\expandafter\def\csname ATRClaim@tool.test.failure_as_zero_accuracy@ciLow\endcsname{null}
\expandafter\def\csname ATRClaim@tool.test.failure_as_zero_accuracy@ciHigh\endcsname{null}
\expandafter\def\csname ATRClaim@tool.test.failure_as_zero_accuracy@diagnostics\endcsname{\{"n\_completed":109,"n\_correct":87,"n\_exec\_failures":0,"n\_unsuccessful":1\}}
\expandafter\def\csname ATRClaim@tool.test.failure_as_zero_accuracy@unit\endcsname{null}
\expandafter\def\csname ATRClaim@tool.test.failure_as_zero_accuracy@scope\endcsname{null}
\expandafter\def\csname ATRClaim@tool.test.failure_as_zero_accuracy@diagnostic@n_completed\endcsname{109}
\expandafter\def\csname ATRClaim@tool.test.failure_as_zero_accuracy@diagnostic@n_correct\endcsname{87}
\expandafter\def\csname ATRClaim@tool.test.failure_as_zero_accuracy@diagnostic@n_exec_failures\endcsname{0}
\expandafter\def\csname ATRClaim@tool.test.failure_as_zero_accuracy@diagnostic@n_unsuccessful\endcsname{1}
\expandafter\def\csname ATRClaim@umbc.equalbyte.200k.hll@value\endcsname{0.014884951539691223}
\expandafter\def\csname ATRClaim@umbc.equalbyte.200k.hll@n\endcsname{100}
\expandafter\def\csname ATRClaim@umbc.equalbyte.200k.hll@windows\endcsname{null}
\expandafter\def\csname ATRClaim@umbc.equalbyte.200k.hll@ciLow\endcsname{0.011952178391387902}
\expandafter\def\csname ATRClaim@umbc.equalbyte.200k.hll@ciHigh\endcsname{0.01620597284352431}
\expandafter\def\csname ATRClaim@umbc.equalbyte.200k.hll@diagnostics\endcsname{\{\}}
\expandafter\def\csname ATRClaim@umbc.equalbyte.200k.hll@unit\endcsname{null}
\expandafter\def\csname ATRClaim@umbc.equalbyte.200k.hll@scope\endcsname{null}
\expandafter\def\csname ATRClaim@umbc.equalbyte.200k.official_umbc@value\endcsname{0.26946666493075716}
\expandafter\def\csname ATRClaim@umbc.equalbyte.200k.official_umbc@n\endcsname{100}
\expandafter\def\csname ATRClaim@umbc.equalbyte.200k.official_umbc@windows\endcsname{null}
\expandafter\def\csname ATRClaim@umbc.equalbyte.200k.official_umbc@ciLow\endcsname{0.22757437763743754}
\expandafter\def\csname ATRClaim@umbc.equalbyte.200k.official_umbc@ciHigh\endcsname{0.3345621455274712}
\expandafter\def\csname ATRClaim@umbc.equalbyte.200k.official_umbc@diagnostics\endcsname{\{\}}
\expandafter\def\csname ATRClaim@umbc.equalbyte.200k.official_umbc@unit\endcsname{null}
\expandafter\def\csname ATRClaim@umbc.equalbyte.200k.official_umbc@scope\endcsname{null}
\expandafter\def\csname ATRClaim@umbc.equalbyte.200k.paired@value\endcsname{0.25488289356874294}
\expandafter\def\csname ATRClaim@umbc.equalbyte.200k.paired@n\endcsname{100}
\expandafter\def\csname ATRClaim@umbc.equalbyte.200k.paired@windows\endcsname{null}
\expandafter\def\csname ATRClaim@umbc.equalbyte.200k.paired@ciLow\endcsname{0.20858350708653664}
\expandafter\def\csname ATRClaim@umbc.equalbyte.200k.paired@ciHigh\endcsname{0.32372132847008234}
\expandafter\def\csname ATRClaim@umbc.equalbyte.200k.paired@diagnostics\endcsname{\{\}}
\expandafter\def\csname ATRClaim@umbc.equalbyte.200k.paired@unit\endcsname{null}
\expandafter\def\csname ATRClaim@umbc.equalbyte.200k.paired@scope\endcsname{null}
\expandafter\def\csname ATRClaim@umbc.equalbyte.50k.hll@value\endcsname{0.014865555525982559}
\expandafter\def\csname ATRClaim@umbc.equalbyte.50k.hll@n\endcsname{100}
\expandafter\def\csname ATRClaim@umbc.equalbyte.50k.hll@windows\endcsname{null}
\expandafter\def\csname ATRClaim@umbc.equalbyte.50k.hll@ciLow\endcsname{0.012267956610292495}
\expandafter\def\csname ATRClaim@umbc.equalbyte.50k.hll@ciHigh\endcsname{0.015954192438634753}
\expandafter\def\csname ATRClaim@umbc.equalbyte.50k.hll@diagnostics\endcsname{\{\}}
\expandafter\def\csname ATRClaim@umbc.equalbyte.50k.hll@unit\endcsname{null}
\expandafter\def\csname ATRClaim@umbc.equalbyte.50k.hll@scope\endcsname{null}
\expandafter\def\csname ATRClaim@umbc.equalbyte.50k.official_umbc@value\endcsname{0.2677574081333674}
\expandafter\def\csname ATRClaim@umbc.equalbyte.50k.official_umbc@n\endcsname{100}
\expandafter\def\csname ATRClaim@umbc.equalbyte.50k.official_umbc@windows\endcsname{null}
\expandafter\def\csname ATRClaim@umbc.equalbyte.50k.official_umbc@ciLow\endcsname{0.21319349376260607}
\expandafter\def\csname ATRClaim@umbc.equalbyte.50k.official_umbc@ciHigh\endcsname{0.3345869748116911}
\expandafter\def\csname ATRClaim@umbc.equalbyte.50k.official_umbc@diagnostics\endcsname{\{\}}
\expandafter\def\csname ATRClaim@umbc.equalbyte.50k.official_umbc@unit\endcsname{null}
\expandafter\def\csname ATRClaim@umbc.equalbyte.50k.official_umbc@scope\endcsname{null}
\expandafter\def\csname ATRClaim@umbc.equalbyte.50k.paired@value\endcsname{0.2505349294856195}
\expandafter\def\csname ATRClaim@umbc.equalbyte.50k.paired@n\endcsname{100}
\expandafter\def\csname ATRClaim@umbc.equalbyte.50k.paired@windows\endcsname{null}
\expandafter\def\csname ATRClaim@umbc.equalbyte.50k.paired@ciLow\endcsname{0.19426897670123172}
\expandafter\def\csname ATRClaim@umbc.equalbyte.50k.paired@ciHigh\endcsname{0.30925462181815966}
\expandafter\def\csname ATRClaim@umbc.equalbyte.50k.paired@diagnostics\endcsname{\{\}}
\expandafter\def\csname ATRClaim@umbc.equalbyte.50k.paired@unit\endcsname{null}
\expandafter\def\csname ATRClaim@umbc.equalbyte.50k.paired@scope\endcsname{null}
\expandafter\def\csname ATRClaim@umbc.equalbyte.8k.hll@value\endcsname{0.01292157688612857}
\expandafter\def\csname ATRClaim@umbc.equalbyte.8k.hll@n\endcsname{100}
\expandafter\def\csname ATRClaim@umbc.equalbyte.8k.hll@windows\endcsname{null}
\expandafter\def\csname ATRClaim@umbc.equalbyte.8k.hll@ciLow\endcsname{0.010842649666683649}
\expandafter\def\csname ATRClaim@umbc.equalbyte.8k.hll@ciHigh\endcsname{0.015436756256424966}
\expandafter\def\csname ATRClaim@umbc.equalbyte.8k.hll@diagnostics\endcsname{\{\}}
\expandafter\def\csname ATRClaim@umbc.equalbyte.8k.hll@unit\endcsname{null}
\expandafter\def\csname ATRClaim@umbc.equalbyte.8k.hll@scope\endcsname{null}
\expandafter\def\csname ATRClaim@umbc.equalbyte.8k.official_umbc@value\endcsname{0.27599702946175064}
\expandafter\def\csname ATRClaim@umbc.equalbyte.8k.official_umbc@n\endcsname{100}
\expandafter\def\csname ATRClaim@umbc.equalbyte.8k.official_umbc@windows\endcsname{null}
\expandafter\def\csname ATRClaim@umbc.equalbyte.8k.official_umbc@ciLow\endcsname{0.24362062609312238}
\expandafter\def\csname ATRClaim@umbc.equalbyte.8k.official_umbc@ciHigh\endcsname{0.33803882039612076}
\expandafter\def\csname ATRClaim@umbc.equalbyte.8k.official_umbc@diagnostics\endcsname{\{\}}
\expandafter\def\csname ATRClaim@umbc.equalbyte.8k.official_umbc@unit\endcsname{null}
\expandafter\def\csname ATRClaim@umbc.equalbyte.8k.official_umbc@scope\endcsname{null}
\expandafter\def\csname ATRClaim@umbc.equalbyte.8k.paired@value\endcsname{0.26689668662792965}
\expandafter\def\csname ATRClaim@umbc.equalbyte.8k.paired@n\endcsname{100}
\expandafter\def\csname ATRClaim@umbc.equalbyte.8k.paired@windows\endcsname{null}
\expandafter\def\csname ATRClaim@umbc.equalbyte.8k.paired@ciLow\endcsname{0.2235347733114689}
\expandafter\def\csname ATRClaim@umbc.equalbyte.8k.paired@ciHigh\endcsname{0.32414264699345086}
\expandafter\def\csname ATRClaim@umbc.equalbyte.8k.paired@diagnostics\endcsname{\{\}}
\expandafter\def\csname ATRClaim@umbc.equalbyte.8k.paired@unit\endcsname{null}
\expandafter\def\csname ATRClaim@umbc.equalbyte.8k.paired@scope\endcsname{null}
\expandafter\def\csname ATRClaim@umbc.resource.hll_resident_state_bytes@value\endcsname{2048}
\expandafter\def\csname ATRClaim@umbc.resource.hll_resident_state_bytes@n\endcsname{null}
\expandafter\def\csname ATRClaim@umbc.resource.hll_resident_state_bytes@windows\endcsname{null}
\expandafter\def\csname ATRClaim@umbc.resource.hll_resident_state_bytes@ciLow\endcsname{null}
\expandafter\def\csname ATRClaim@umbc.resource.hll_resident_state_bytes@ciHigh\endcsname{null}
\expandafter\def\csname ATRClaim@umbc.resource.hll_resident_state_bytes@diagnostics\endcsname{\{\}}
\expandafter\def\csname ATRClaim@umbc.resource.hll_resident_state_bytes@unit\endcsname{bytes}
\expandafter\def\csname ATRClaim@umbc.resource.hll_resident_state_bytes@scope\endcsname{null}
\expandafter\def\csname ATRClaim@umbc.resource.shared_frozen_feature_trunk_bytes@value\endcsname{8394880}
\expandafter\def\csname ATRClaim@umbc.resource.shared_frozen_feature_trunk_bytes@n\endcsname{null}
\expandafter\def\csname ATRClaim@umbc.resource.shared_frozen_feature_trunk_bytes@windows\endcsname{null}
\expandafter\def\csname ATRClaim@umbc.resource.shared_frozen_feature_trunk_bytes@ciLow\endcsname{null}
\expandafter\def\csname ATRClaim@umbc.resource.shared_frozen_feature_trunk_bytes@ciHigh\endcsname{null}
\expandafter\def\csname ATRClaim@umbc.resource.shared_frozen_feature_trunk_bytes@diagnostics\endcsname{\{\}}
\expandafter\def\csname ATRClaim@umbc.resource.shared_frozen_feature_trunk_bytes@unit\endcsname{bytes}
\expandafter\def\csname ATRClaim@umbc.resource.shared_frozen_feature_trunk_bytes@scope\endcsname{null}
\expandafter\def\csname ATRClaim@umbc.resource.umbc_carried_state_bytes@value\endcsname{2048}
\expandafter\def\csname ATRClaim@umbc.resource.umbc_carried_state_bytes@n\endcsname{null}
\expandafter\def\csname ATRClaim@umbc.resource.umbc_carried_state_bytes@windows\endcsname{null}
\expandafter\def\csname ATRClaim@umbc.resource.umbc_carried_state_bytes@ciLow\endcsname{null}
\expandafter\def\csname ATRClaim@umbc.resource.umbc_carried_state_bytes@ciHigh\endcsname{null}
\expandafter\def\csname ATRClaim@umbc.resource.umbc_carried_state_bytes@diagnostics\endcsname{\{\}}
\expandafter\def\csname ATRClaim@umbc.resource.umbc_carried_state_bytes@unit\endcsname{bytes}
\expandafter\def\csname ATRClaim@umbc.resource.umbc_carried_state_bytes@scope\endcsname{null}
\expandafter\def\csname ATRClaim@umbc.resource.umbc_static_parameter_bytes@value\endcsname{66868}
\expandafter\def\csname ATRClaim@umbc.resource.umbc_static_parameter_bytes@n\endcsname{null}
\expandafter\def\csname ATRClaim@umbc.resource.umbc_static_parameter_bytes@windows\endcsname{null}
\expandafter\def\csname ATRClaim@umbc.resource.umbc_static_parameter_bytes@ciLow\endcsname{null}
\expandafter\def\csname ATRClaim@umbc.resource.umbc_static_parameter_bytes@ciHigh\endcsname{null}
\expandafter\def\csname ATRClaim@umbc.resource.umbc_static_parameter_bytes@diagnostics\endcsname{\{\}}
\expandafter\def\csname ATRClaim@umbc.resource.umbc_static_parameter_bytes@unit\endcsname{bytes}
\expandafter\def\csname ATRClaim@umbc.resource.umbc_static_parameter_bytes@scope\endcsname{null}

\providecommand{\GemmaScaleN}{3969}
\providecommand{\GemmaScaleWindows}{174}
\providecommand{\GemmaScaleBudgetedValue}{0.9916855631141346}
\providecommand{\GemmaScaleBudgetedLow}{0.9872148408122337}
\providecommand{\GemmaScaleBudgetedHigh}{0.9952404809619239}
\providecommand{\GemmaScaleExactValue}{1.0}
\providecommand{\GemmaScaleCounterfactualValue}{1.0}
\providecommand{\GemmaScaleAggregateOnlyValue}{0.5137314184933233}
\providecommand{\GemmaScaleAggregateOnlyLow}{0.5010368066355625}
\providecommand{\GemmaScaleAggregateOnlyHigh}{0.5257678292555961}
\providecommand{\GemmaScalePolicyOnlyValue}{0.27286470143613}
\providecommand{\GemmaScalePolicyOnlyLow}{0.249686952166291}
\providecommand{\GemmaScalePolicyOnlyHigh}{0.2972404730617608}
\providecommand{\GemmaScaleBudgetedMinusExactValue}{-0.008314436885865457}
\providecommand{\GemmaScaleBudgetedMinusExactLow}{-0.012785159187766357}
\providecommand{\GemmaScaleBudgetedMinusExactHigh}{-0.004759519038076153}

\providecommand{\QwenCodeAccuracy}{1.0}
\providecommand{\QwenCodeMinusBudgeted}{0.08916666666666667}
\providecommand{\QwenCodeMinusBudgetedLow}{0.07142857142857142}
\providecommand{\QwenCodeMinusBudgetedHigh}{0.1071435932163467}
\providecommand{\QwenCodeCallsPerTask}{1.01}
\providecommand{\QwenCodeOutputTokensPerTask}{1260.2016666666666}

\providecommand{\GemmaCodeAccuracy}{1.0}
\providecommand{\GemmaCodeMinusBudgeted}{0.0075}
\providecommand{\GemmaCodeMinusBudgetedLow}{0.002514668901927913}
\providecommand{\GemmaCodeMinusBudgetedHigh}{0.013289036544850499}
\providecommand{\GemmaCodeCallsPerTask}{1.0}
\providecommand{\GemmaCodeOutputTokensPerTask}{208.5575}

\providecommand{\MillionRecordMeanRelativeErrorPct}{1.5677633442786227}
\providecommand{\MillionRecordEstimateCount}{20}

\providecommand{\ATRPct}[1]{\fpeval{round(100*(#1),2)}\%}
\providecommand{\ATRRoundTwo}[1]{\fpeval{round(#1,2)}}
\providecommand{\ATRPctClaim}[1]{\ATRPct{\ATRClaimValue{#1}}}
\providecommand{\ATRPctCI}[1]{%
  \ATRPct{\ATRClaimValue{#1}}
  [\ATRPct{\ATRClaimCILow{#1}}, \ATRPct{\ATRClaimCIHigh{#1}}]}
\appendix

\section{Normalized Attention and the Aggregation-Operator Mismatch}
\label{app:proofs}

In the paper, we stated that ``a normalized readout does not
itself expose the update and merge rules needed for set-based aggregation.''
This statement concerns the state returned by the readout, not the computational power
of a Transformer, an external executor, or a model supplied with an additional
aggregation state.

To further justify this, consider the records in one stream $\mathcal R^{(i)}$. Let $v_t$ denote
the model representation associated with record $t$, and let $a_t$ be its
attention score at the readout position. A normalized attention readout can be
written as
\begin{equation}
o
=
\frac{N}{D},
\qquad
N=\sum_{t=1}^{n_i}e^{a_t}v_t,
\qquad
D=\sum_{t=1}^{n_i}e^{a_t}.
\label{eq:app_normalized_readout}
\end{equation}
The notations in this section are local to the analysis; they are not additional
components of our method.

\begin{observation}[Effect of a repeated record]
\label{prop:duplicate_effect}
Suppose another record for an identity already present in
$\mathcal R^{(i)}$ enters the readout with value $v$ and positive weight
$w=e^a$. If $o'$ is the updated output, then
\begin{equation}
o'-o
=
\frac{w}{D+w}(v-o).
\label{eq:app_duplicate_effect}
\end{equation}
The output is unchanged if and only if $v=o$.
\end{observation}

\begin{proof}
The new output is $o'=(N+wv)/(D+w)$. Subtracting $N/D$ and using
$N=Do$ gives Equation~\eqref{eq:app_duplicate_effect}. Since $w>0$,
$o'=o$ holds exactly when $v=o$.
\end{proof}

Set-based aggregation requires a different rule. Once an identity has entered
the set, seeing it again must leave the represented set unchanged. Normalized
attention has no such guarantee: the second record remains another term in the
average. This loss of cardinality information in attention-based aggregation
has also been studied in graph neural networks~\citep{zhang2020improving}.

\begin{observation}[Normalized outputs do not determine a merged output]
\label{prop:no_normalized_merge}
No binary function can recover the normalized output of every concatenated
pair of record segments from their two normalized outputs alone.
\end{observation}

\begin{proof}
Equal attention scores reduce Equation~\eqref{eq:app_normalized_readout} to an
arithmetic mean. Take the scalar segments
\[
\mathcal R_A=(0,2),
\qquad
\mathcal R_B=(1),
\qquad
\mathcal T=(0).
\]
Both $\mathcal R_A$ and $\mathcal R_B$ have mean $1$, while $\mathcal T$ has
mean $0$. Their concatenations do not agree:
\[
o(\mathcal R_A\Vert\mathcal T)=\frac{2}{3},
\qquad
o(\mathcal R_B\Vert\mathcal T)=\frac{1}{2}.
\]
A merge function receives the same input pair $(1,0)$ in both cases and
therefore cannot return both answers.
\end{proof}

Based on the above observations, keeping $(N,D)$ would make the average mergeable because the numerators and
denominators can be added. It would not make the update duplicate-idempotent:
an identity present in two segments would still contribute twice. The result
would be a mergeable weighted average, not set union. We use
\emph{aggregation-operator mismatch} for this difference between the
update-and-merge algebra of normalized averaging and that required by the
target aggregation.

\paragraph{Scope of the theoretical contribution.}
The two observations above and Proposition~\ref{prop:hll_segment_composition}
are the complete set of new theoretical statements made in this paper. Their
assumptions, restrictions, and proofs are stated in this appendix. The
cardinality estimator, Count--Min bound, and JMLE likelihood used later are
established tools and are cited rather than presented as new theorems.

\section{Notation and Exact Targets}
\label{app:operator_details}

\subsection{Notation shared with the main paper}

Table~\ref{tab:app_notation} repeats the notation needed to read the
appendix. Notations introduced only inside a proof or estimator are not included and are defined
where they first appear.

\begin{table}[ht]
\centering
\small
\caption{Notation used in the main paper and appendix.}
\label{tab:app_notation}
\begin{tabular}{L{0.25\textwidth}L{0.67\textwidth}}
\toprule
Notation & Meaning \\
\midrule
$x_{1:L}$, $q$ & context of length $L$ and the request over that context \\
$E$ & extractor that organizes relevant records into streams \\
$\tau$ & number of streams required by request $q$ \\
$\mathcal R^{(i)}$ & stream $i$, containing $n_i$ records \\
$z_t^{(i)}$ & canonical identity of record $t$ in stream $i$ \\
$g_t^{(i)}$ & optional group key of that record \\
$\mathcal U$, $\mathcal G$ & identity universe and group-key universe \\
$s_i\in\mathcal S$ & maintained state for stream $i$ and its state space \\
$f_q$, $y_q$ & exact aggregation function and exact answer \\
$\rho_q$, $\widehat y_q$ & request-specific readout and its estimate \\
$m$ & number of one-byte registers in one HLL state \\
$P_i$ & number of separately processed segments of stream $i$ \\
$d_{\mathrm{cm}}$, $w_{\mathrm{cm}}$ & Count--Min rows and counters per row \\
\bottomrule
\end{tabular}
\end{table}

\subsection{Target aggregations and empty-set conventions}

For each stream, define the set of observed identities and the
identity set within group $g$ as
\begin{equation}
A_i
=
\{z_t^{(i)}:1\leq t\leq n_i\},
\qquad
A_{i,g}
=
\{z_t^{(i)}:g_t^{(i)}=g,\ 1\leq t\leq n_i\}.
\label{eq:app_identity_sets}
\end{equation}
The number of occurrences of identity $z$ in the same stream is
\begin{equation}
\mu_i(z)
=
\bigl|\{t:z_t^{(i)}=z,\ 1\leq t\leq n_i\}\bigr|.
\label{eq:app_multiplicity}
\end{equation}

The exact targets used in this paper follow directly:
\begin{align}
f_{\mathrm{distinct}}(\mathcal R^{(i)})
&=|A_i|, \nonumber\\
f_{\mathrm{union}}(\mathcal R^{(1)},\mathcal R^{(2)})
&=|A_1\cup A_2|, \nonumber\\
f_{\mathrm{Jaccard}}(\mathcal R^{(1)},\mathcal R^{(2)})
&=\frac{|A_1\cap A_2|}{|A_1\cup A_2|}, \nonumber\\
f_{\mathrm{contain}}(\mathcal R^{(1)},\mathcal R^{(2)})
&=\frac{|A_1\cap A_2|}{|A_1|}, \nonumber\\
f_{\mathrm{group}(i,g)}(\mathcal R^{(i)})
&=|A_{i,g}|, \nonumber\\
f_{\mathrm{freq}(i,z^\star)}(\mathcal R^{(i)})
&=\mu_i(z^\star).
\label{eq:app_exact_targets}
\end{align}
Distinct count and grouped distinct count are zero for an empty input or
group. We set Jaccard similarity to zero when $A_1\cup A_2$ is empty and
containment to zero when $A_1$ is empty. The exact targets and approximate
readouts use the same conventions.

\section{HLL State Construction and Readout}
\label{app:hll_details}

\subsection{Hashing and register updates}

One HLL state for stream $i$ is
\[
s_i=(M_{i,1},\ldots,M_{i,m}),
\]
with every register initialized to zero. The implementation hashes each
canonical identity to 64 bits using a fixed seed. For
$p=\log_2m$, $p$ hash bits select a register $j(z)$; the remaining bits
determine the rank $r(z)$, defined as one plus their number of leading zeros.
The finite 64-bit hash caps the rank at $64-p+1$. Processing identity $z$
changes one register:
\begin{equation}
M_{i,j(z)}
\leftarrow
\max\{M_{i,j(z)},r(z)\}.
\label{eq:app_hll_update}
\end{equation}
The original identity is not retained. Because a fixed hash and seed map every
occurrence of $z$ to the same pair $(j(z),r(z))$, applying
Equation~\eqref{eq:app_hll_update} again has no effect once the register has
reached that rank.

For a group-by request, $g_t^{(i)}$ first selects the state $s_{i,g_t^{(i)}}$.
The identity $z_t^{(i)}$ then performs the same one-register update. An
identity may therefore appear in several groups, while repeated occurrences
within one group remain idempotent.

\subsection{Segment composition}

Let $s_i^{[a]}=(M_{i,1}^{[a]},\ldots,M_{i,m}^{[a]})$ be the state built from
segment $a$ of stream $i$. States are compatible when they use the
same register count, hash function, and seed. Their merge is the registerwise
maximum:
\begin{equation}
s_i
=
\bigoplus_{a=1}^{P_i}s_i^{[a]},
\qquad
M_{i,j}
=
\max_{1\leq a\leq P_i}M_{i,j}^{[a]}.
\label{eq:app_hll_segment_merge}
\end{equation}

\begin{proposition}[Segment composition]
\label{prop:hll_segment_composition}
Equation~\eqref{eq:app_hll_segment_merge} produces exactly the same register
array as processing the complete stream once.
\end{proposition}

\begin{proof}
For register $j$, each segment stores the largest rank among identities in
that segment that map to $j$, or zero if no such identity occurs. Taking the
maximum over segments is therefore the largest rank among all identities in
the complete stream that map to $j$. This is precisely the value produced by
a single pass. The argument holds independently for every register.
\end{proof}

The equality is at the state level; the cardinality readout remains an
estimate. States from different streams are not combined during
construction, because relation readouts must still know which stream
contributed each register array.

\subsection{Cardinality and union}

For the register counts evaluated in this paper ($m\geq256$), let
\begin{equation}
\widetilde F_0(s_i)
=
\alpha_m m^2
\left(\sum_{j=1}^{m}2^{-M_{i,j}}\right)^{-1},
\qquad
\alpha_m=\frac{0.7213}{1+1.079/m}.
\label{eq:app_hll_raw_estimator}
\end{equation}
If $\widetilde F_0(s_i)\leq2.5m$ and $V_i>0$, where $V_i$ is the number of
zero registers, the implementation uses the small-range correction
\begin{equation}
\widehat F_0(s_i)
=
m\log\frac{m}{V_i}.
\label{eq:app_hll_linear_counting}
\end{equation}
Otherwise, $\widehat F_0(s_i)=\widetilde F_0(s_i)$. These are the standard HLL
readouts~\citep{hll2007}.

A distinct-count request returns $\widehat F_0(s_i)$. A union request forms
$s_1\oplus s_2$ temporarily and returns
\begin{equation}
\widehat y_q
=
\rho_q(s_1,s_2)
=
\widehat F_0(s_1\oplus s_2).
\label{eq:app_hll_union}
\end{equation}
The temporary merge does not overwrite either state.

\subsection{Jaccard similarity and containment}
\label{app:jmle}

Jaccard similarity and containment read
the two compatible HLL arrays jointly. Let $A=A_1$ and $B=A_2$, and define
\[
n_{10}=|A\setminus B|,
\qquad
n_{01}=|B\setminus A|,
\qquad
n_{11}=|A\cap B|.
\]
The joint maximum-likelihood estimator (JMLE) estimates these three regions
directly from the paired registers~\citep{ertl2017new}.

The following likelihood records the implementation used in our experiments.
Let $Q=64-\log_2m$ be the number of rank bits and define
\begin{equation}
w(k)=
\begin{cases}
2^{-k}, & 0\leq k\leq Q,\\
0, & k\geq Q+1.
\end{cases}
\label{eq:app_jmle_weight}
\end{equation}
The second branch handles the saturated register value $Q+1$ exactly. Under
the standard Poissonized HLL model, the joint cumulative probability of a
paired register is
\begin{equation}
F(u,v;\boldsymbol n)
=
\exp\!\left[
-\frac{
n_{10}w(u)+n_{01}w(v)+n_{11}w(\min\{u,v\})
}{m}
\right],
\label{eq:app_jmle_cdf}
\end{equation}
where $\boldsymbol n=(n_{10},n_{01},n_{11})$ and $F(u,v;\boldsymbol n)=0$ if
$u<0$ or $v<0$. The probability of observing the register pair $(u,v)$ is
the two-dimensional finite difference
\begin{align}
p(u,v;\boldsymbol n)
={}&F(u,v;\boldsymbol n)-F(u-1,v;\boldsymbol n) \nonumber\\
&-F(u,v-1;\boldsymbol n)+F(u-1,v-1;\boldsymbol n).
\label{eq:app_jmle_pmf}
\end{align}
If $H_{uv}$ counts how many of the $m$ register pairs equal $(u,v)$, JMLE
maximizes
\begin{equation}
\ell(\boldsymbol n)
=
\sum_{u,v}H_{uv}\log p(u,v;\boldsymbol n)
\quad\text{subject to}\quad
n_{10},n_{01},n_{11}\geq0.
\label{eq:app_jmle_likelihood}
\end{equation}

The implementation performs deterministic coordinate ascent in log space and
tests zero explicitly for every coordinate. A candidate step is accepted only
when it does not decrease Equation~\eqref{eq:app_jmle_likelihood}. Convergence
requires both parameter stability and the boundary-aware first-order
condition: interior coordinates must have negligible gradient, while a
coordinate at zero must have no improving feasible direction. The optimizer
returns its termination reason, iteration count, boundary coordinates, and
optimality residual. If these checks fail, the relation readout is marked
invalid; it is not replaced by inclusion--exclusion.

With the resulting estimates
$(\widehat n_{10},\widehat n_{01},\widehat n_{11})$, the readouts are
\begin{equation}
\widehat J
=
\frac{\widehat n_{11}}
{\widehat n_{10}+\widehat n_{01}+\widehat n_{11}},
\qquad
\widehat C_{A\to B}
=
\frac{\widehat n_{11}}
{\widehat n_{10}+\widehat n_{11}}.
\label{eq:app_jmle_readouts}
\end{equation}
The zero-denominator cases follow the conventions in
Section~\ref{app:operator_details}. Both ratios use the same two
states; no additional persistent set state is allocated.

\section{Frequency-State Extension}
\label{app:frequency_details}

HLL deliberately ignores repeated identities. If request $q$ asks for the
frequency of one identity $z^\star$ known before processing, an exact scalar
counter is enough:
\begin{equation}
c\leftarrow c+\mathbf 1[z=z^\star].
\label{eq:app_scalar_counter}
\end{equation}
Scalar states built from separate segments merge by addition.

When the queried identity is selected after processing, or when several
identities may be queried, Count--Min Sketch maintains
$C\in\mathbb N_0^{d_{\mathrm{cm}}\times w_{\mathrm{cm}}}$~\citep{countmin2004}.
Let $\eta_\ell$ be the hash function for row $\ell$. Each occurrence updates
one counter per row, and a point query returns the smallest corresponding
counter:
\begin{equation}
\begin{aligned}
C_{\ell,\eta_\ell(z)}
&\leftarrow C_{\ell,\eta_\ell(z)}+1,
&&\ell=1,\ldots,d_{\mathrm{cm}},\\
\widehat\mu_i(z^\star)
&=
\min_{1\leq\ell\leq d_{\mathrm{cm}}}
C_{\ell,\eta_\ell(z^\star)}.
\end{aligned}
\label{eq:app_countmin}
\end{equation}
Compatible tables use the same dimensions and row hashes and merge by
elementwise addition. If $N_i$ is the number of updates to stream $i$,
the standard hashing assumptions give
\begin{equation}
\mu_i(z)
\leq
\widehat\mu_i(z)
\leq
\mu_i(z)+\frac{eN_i}{w_{\mathrm{cm}}}
\label{eq:app_countmin_bound}
\end{equation}
with probability at least $1-e^{-d_{\mathrm{cm}}}$. The table size is fixed,
but the collision term grows with stream mass. Count--Min is included as a
frequency extension; the primary set-based experiments use HLL.

\section{Experimental Protocol}
\label{app:experimental_protocol}

\subsection{Cohorts and comparison arms}

The aggregate--then--reason evaluation fixes the structured records and
precompiled operator specification before model inference. This isolates the
question studied here: whether a fixed-budget aggregation state supplies the
evidence needed for a later decision. It does not evaluate extraction from raw
prose.

The large-scale cohort contains 3,969 unique tasks from 174 source windows
and is evaluated in full by both Qwen and Gemma. A fixed 1,200-task subset
spanning the same 174 windows is used for the same-cohort code-execution
comparison. The full-context comparison selects one task per window without
looking at its answer, giving 174 matched tasks per model. All full-context
methods use the same served 150,000-token limit.

Four state-handoff arms separate approximation error from missing decision
information. The \emph{exact} arm supplies the exact aggregate and the
decision condition. The \emph{budgeted} arm replaces the exact aggregate with
the fixed-budget readout. The \emph{aggregate-only} and \emph{policy-only}
controls withhold the decision condition and the aggregate, respectively.
The full-context comparison uses direct prompting, chain-of-thought prompting,
and guided-choice prompting on the matched raw context.

\subsection{Data provenance and deterministic task construction}

The experiments introduce no new raw corpus. They use the public
Oolong-Synth dataset~\citep{bertsch2025oolong}, available from
\url{https://huggingface.co/datasets/oolongbench/oolong-synth}. The large-scale
campaign loads the dataset at revision
\texttt{f0d59eaf0febf130664cfceb710436c8e3216b2b}. The released structured
records are consumed directly: there is no imputation, learned preprocessing,
token filtering, or data augmentation. Canonical identities are converted to
UTF-8 strings before fixed-seed hashing.

Generator settings were fixed using 30 validation windows and then applied
once to the test split. The final values were a maximum of 30
candidate items per source window, a 16-item minimum, a 0.60 non-additive
fraction, a target of 26 tasks per window, and 16 generation attempts per
operator. The generator seed was 20,260,720. Test answers were not used to
select these settings or the 1,200-task replication subset.

\subsection{Final settings and development ranges}

Table~\ref{tab:app_final_settings} lists the final settings used by the
paper-facing runs. The main HLL size was fixed before test evaluation. The
controlled budget study additionally evaluated
$m\in\{256,512,1024,2048,4096,8192\}$ without changing the main-run setting.
UMBC architecture selection used validation streams only; test streams were
opened after the winning architecture was fixed.

\begin{table}[ht]
\centering
\small
\setlength{\tabcolsep}{5pt}
\caption{Final experimental settings. ``TP'' denotes tensor parallelism.}
\label{tab:app_final_settings}
\begin{tabular}{L{0.27\textwidth}L{0.66\textwidth}}
\toprule
Component & Final setting \\
\midrule
HLL &
$m=2048$ one-byte registers; 64-bit keyed hash; seed 0; register-max merge \\
HLL JMLE &
24 coordinate sweeps; parameter tolerance $10^{-3}$; KKT tolerance
$5\times10^{-4}$; no inclusion--exclusion fallback \\
Count--Min &
4 rows $\times$ 64 columns of int64 cells (2,048 bytes per table);
label seed 0 and user seed 1 \\
Qwen reasoner &
Qwen3.6-35B-A3B revision
\texttt{995ad96eacd98c81ed38be0c5b274b04031597b0}; BF16; temperature 0;
test seed 20,260,718; 96 output tokens for state handoff; 8,192 direct and
16,384 CoT output-token caps; 150,000 served tokens; TP2 \\
Gemma reasoner &
Gemma-4-31B-IT revision
\texttt{3548789868c5356dbf307c98e6f609007b82b3eb}; BF16; temperature 0;
test seed 20,260,721; 96 output tokens for state handoff and a 150,000-token
served limit; direct MLX serving with \texttt{mlx-lm} 0.31.3; no quantization \\
Code execution &
The Qwen and Gemma revisions above; temperature 0; at most 4,096 generated
tokens per round, three rounds, and 60 seconds of isolated Python execution
per round; Apptainer isolation for Qwen and macOS Seatbelt isolation for
Gemma \\
Official UMBC &
2,048-byte carried state; deterministic 16-slot state with
$\widehat d=28$, four heads, and slot-sigmoid attention; AdamW,
learning rate $3\times10^{-3}$, 4,000 final steps; training seeds
$\{0,1,2\}$ \\
Uncertainty &
10,000 source-window-cluster bootstrap resamples; analysis seeds 20,260,720
and 20,260,721; paired randomization seed 20,260,728 \\
\bottomrule
\end{tabular}
\end{table}

\subsection{Run counts}

Every reported model--task--arm cell is one deterministic inference run at
temperature zero; rows from repeated arms are not counted as additional
independent tasks. The Qwen and Gemma scale experiments each evaluate the same
3,969 tasks from 174 windows once in each of six prespecified conditions
(23,814 terminal rows per model). The balanced full-context experiment uses
174 tasks per prompt condition and model. The code-execution reference uses
the fixed 1,200-task subset once per model (2,400 terminal task runs in total).
Failed calls or parses remain in their corresponding denominator with score
zero.

\subsection{Counterfactual mechanism check}

The counterfactual arm changes the supplied aggregate and recomputes the answer
that follows from that change. It therefore checks whether the reasoner uses
the aggregate evidence rather than reproducing the original answer. Across
\ATRClaimN{scale.counterfactual.accuracy} tasks from
\ATRClaimWindows{scale.counterfactual.accuracy} source windows, counterfactual
accuracy is \ATRPctCI{scale.counterfactual.accuracy}. This is a mechanism
diagnostic, not accuracy on the original task.

\subsection{Common-packing-eligible sensitivity}

The primary full-context comparison retains all 174 tasks, including failures.
As a secondary check, we intersect the tasks that fit under direct,
chain-of-thought, and guided prompting for both models. This answer-blind rule
leaves 94 tasks. Parsing failures would remain in the denominator, although
none occur in the three values shown in Table~\ref{tab:app_common_eligible}.

\begin{table}[ht]
\centering
\small
\setlength{\tabcolsep}{5pt}
\caption{Common-packing-eligible sensitivity analysis. Values are task
accuracy or paired percentage-point differences with 95\% source-window-cluster
bootstrap intervals. This selected analysis does not replace the all-window
primary comparison.}
\label{tab:app_common_eligible}
\begin{tabular}{lcccc}
\toprule
Model & $n$ & Budgeted state & Guided context & Difference \\
\midrule
Qwen3.6-35B-A3B &
\ATRClaimN{cross.qwen.commoneligible_budgeted} &
\ATRPctCI{cross.qwen.commoneligible_budgeted} &
\ATRPctCI{cross.qwen.commoneligible_guided} &
\ATRPctCI{cross.qwen.commoneligible_budgeted_minus_guided} \\
Gemma-4-31B-IT &
\ATRClaimN{cross.gemma.commoneligible_budgeted} &
\ATRPctCI{cross.gemma.commoneligible_budgeted} &
\ATRPctCI{cross.gemma.commoneligible_guided} &
\ATRPctCI{cross.gemma.commoneligible_budgeted_minus_guided} \\
\bottomrule
\end{tabular}
\end{table}

\section{Controlled State and Resource Results}
\label{app:controlled_results}

\subsection{Full-scale state-handoff results}

Table~\ref{tab:app_full_scale} reports the complete 3,969-task state-handoff
evaluation for both models. Every condition uses the same task set and source
windows.

\begin{table}[ht]
\centering
\small
\setlength{\tabcolsep}{7pt}
\caption{Accuracy on all 3,969 tasks from 174 source windows. Brackets give
95\% source-window-cluster bootstrap intervals. All failures remain in the
denominator with score zero.}
\label{tab:app_full_scale}
\begin{tabular}{lcc}
\toprule
Condition & Qwen & Gemma \\
\midrule
Fixed-budget state &
\ATRPctCI{scale.budgeted.accuracy} &
\ATRPct{\GemmaScaleBudgetedValue}
[\ATRPct{\GemmaScaleBudgetedLow}, \ATRPct{\GemmaScaleBudgetedHigh}] \\
Exact aggregate &
\ATRPctCI{scale.exact.accuracy} &
\ATRPct{\GemmaScaleExactValue}
[\ATRPct{\GemmaScaleExactValue}, \ATRPct{\GemmaScaleExactValue}] \\
Counterfactual aggregate &
\ATRPctCI{scale.counterfactual.accuracy} &
\ATRPct{\GemmaScaleCounterfactualValue}
[\ATRPct{\GemmaScaleCounterfactualValue},
\ATRPct{\GemmaScaleCounterfactualValue}] \\
Aggregate-only &
\ATRPctCI{scale.aggregate_only.accuracy} &
\ATRPct{\GemmaScaleAggregateOnlyValue}
[\ATRPct{\GemmaScaleAggregateOnlyLow}, \ATRPct{\GemmaScaleAggregateOnlyHigh}] \\
Policy-only &
\ATRPctCI{scale.policy_only.accuracy} &
\ATRPct{\GemmaScalePolicyOnlyValue}
[\ATRPct{\GemmaScalePolicyOnlyLow}, \ATRPct{\GemmaScalePolicyOnlyHigh}] \\
\bottomrule
\end{tabular}
\end{table}

Gemma reaches 99.17\% with the fixed-budget state and 100.00\% with the exact
aggregate. The paired fixed-budget-minus-exact difference is $-0.83$
percentage points (95\% cluster interval, $-1.28$ to $-0.48$). The Gemma
policy-only arm has 1,772 parsing failures; they remain in the denominator.
The other five Gemma arms have no parsing, transport, or model failures.

\subsection{Official UMBC comparison at equal carried-state size}

We compare HLL with the official Universal Mini-Batch Consistency (UMBC)
implementation~\citep{umbc2023}. Both methods carry 2,048 bytes from one
partition to the next and receive the same identity stream. Equal carried
state does not mean equal total memory: UMBC also has learned parameters, and
both methods use the same frozen feature trunk. Table~\ref{tab:app_umbc}
therefore reports these resource cells separately.

\begin{table}[ht]
\centering
\small
\setlength{\tabcolsep}{6pt}
\caption{Median relative error at equal 2,048-byte carried state, using 100
streams per length. Brackets give 95\% bootstrap intervals. The final column
is the paired UMBC-minus-HLL difference.}
\label{tab:app_umbc}
\begin{tabular}{lccc}
\toprule
Length & HLL & Official UMBC & Paired difference \\
\midrule
8K &
\ATRPctCI{umbc.equalbyte.8k.hll} &
\ATRPctCI{umbc.equalbyte.8k.official_umbc} &
\ATRPctCI{umbc.equalbyte.8k.paired} \\
50K &
\ATRPctCI{umbc.equalbyte.50k.hll} &
\ATRPctCI{umbc.equalbyte.50k.official_umbc} &
\ATRPctCI{umbc.equalbyte.50k.paired} \\
200K &
\ATRPctCI{umbc.equalbyte.200k.hll} &
\ATRPctCI{umbc.equalbyte.200k.official_umbc} &
\ATRPctCI{umbc.equalbyte.200k.paired} \\
\bottomrule
\end{tabular}
\end{table}

\begin{table}[ht]
\centering
\small
\caption{Resource accounting for the equal-carried-state comparison. Static
parameters and the shared trunk are not folded into the carried-state cells.}
\label{tab:app_umbc_resources}
\begin{tabular}{L{0.48\textwidth}r}
\toprule
Resource cell & Bytes \\
\midrule
HLL resident register state &
\ATRClaimValue{umbc.resource.hll_resident_state_bytes} \\
UMBC carried state &
\ATRClaimValue{umbc.resource.umbc_carried_state_bytes} \\
UMBC static parameters &
\ATRClaimValue{umbc.resource.umbc_static_parameter_bytes} \\
Shared frozen feature trunk &
\ATRClaimValue{umbc.resource.shared_frozen_feature_trunk_bytes} \\
\bottomrule
\end{tabular}
\end{table}

HLL has lower median error at all three lengths under the equal carried-state
budget. The comparison supports a claim about the state passed across
partitions, not a claim that the two end-to-end systems have identical memory.

\subsection{Register budget and distinct-count error}

Table~\ref{tab:app_hll_budget} gives the complete controlled budget sweep. The
state uses one byte per register, so the storage column is also the serialized
register payload. Increasing $m$ lowers error without changing the update
rule or introducing storage that grows with the stream.

\begin{table}[ht]
\centering
\small
\setlength{\tabcolsep}{8pt}
\caption{Distinct-count RMSE across HLL register budgets. Brackets give 95\%
bootstrap intervals; $n$ is the number of controlled streams at that budget.}
\label{tab:app_hll_budget}
\begin{tabular}{rrrr}
\toprule
Registers $m$ & State bytes & $n$ & Relative-error RMSE \\
\midrule
256  & 256  & \ATRClaimN{package_d.distinct_rmse.m256}
& \ATRPctCI{package_d.distinct_rmse.m256} \\
512  & 512  & \ATRClaimN{package_d.distinct_rmse.m512}
& \ATRPctCI{package_d.distinct_rmse.m512} \\
1,024 & 1,024 & \ATRClaimN{package_d.distinct_rmse.m1024}
& \ATRPctCI{package_d.distinct_rmse.m1024} \\
2,048 & 2,048 & \ATRClaimN{package_d.distinct_rmse.m2048}
& \ATRPctCI{package_d.distinct_rmse.m2048} \\
4,096 & 4,096 & \ATRClaimN{package_d.distinct_rmse.m4096}
& \ATRPctCI{package_d.distinct_rmse.m4096} \\
8,192 & 8,192 & \ATRClaimN{package_d.distinct_rmse.m8192}
& \ATRPctCI{package_d.distinct_rmse.m8192} \\
\bottomrule
\end{tabular}
\end{table}
\FloatBarrier

The 1.6\% value reported in the abstract uses a different, explicitly fixed
summary. At $m=2048$ and one million records per stream, the fixed-data
hash-replicate axis evaluates two streams under ten hash seeds
(20 estimates). Its mean absolute relative error,
$\operatorname{mean}(|\widehat y-y|/y)$, is
\fpeval{round(\MillionRecordMeanRelativeErrorPct,3)}\%, reported as 1.6\%.
In contrast,
Table~\ref{tab:app_hll_budget} reports RMSE over the combined data- and
hash-replicate axes. The two values differ because both the replicate pool and
the error functional differ.

Fixed size does not imply that a sketch is always smaller than an exact
representation. For one $m=2048$ HLL state, the measured crossover against the
serialized UTF-8 identity payload is
\ATRRoundTwo{\ATRClaimValue{package_d.m2048_utf8_crossover}} distinct identities
(95\% CI,
\ATRRoundTwo{\ATRClaimCILow{package_d.m2048_utf8_crossover}}--%
\ATRRoundTwo{\ATRClaimCIHigh{package_d.m2048_utf8_crossover}}). Exact payload is smaller
below this point; HLL is smaller above it. The measurement excludes container,
allocator, and whole-system overhead and must not be read as a Python-set
memory comparison.

\subsection{Same-cohort code-execution reference}

The code-execution arm uses the same 1,200 tasks, structured records, and
operator specifications as the state arms for each model. It generated a
Python program for every task and executed that program in a sandbox.
Table~\ref{tab:app_tool} reports failure-as-zero accuracy and average
generation resources on these common cohorts.

\begin{table}[ht]
\centering
\small
\setlength{\tabcolsep}{6pt}
\caption{Same-cohort comparison on 1,200 tasks and 174 windows per model.
Calls and output tokens are averages per task.}
\label{tab:app_tool}
\begin{tabular}{llrrr}
\toprule
Model & Method & Accuracy & Calls & Output tokens \\
\midrule
Qwen & Fixed-budget state & 91.08\% & 1.00 & 2.2 \\
Qwen & Code execution & 100.00\% & \QwenCodeCallsPerTask &
\fpeval{round(\QwenCodeOutputTokensPerTask,1)} \\
Gemma & Fixed-budget state & 99.25\% & 1.00 & 2.2 \\
Gemma & Code execution & 100.00\% & \GemmaCodeCallsPerTask &
\fpeval{round(\GemmaCodeOutputTokensPerTask,1)} \\
\bottomrule
\end{tabular}
\end{table}

All 1,200 runs completed correctly for each model. Qwen used 1.01 rounds and
1,260.2 output tokens per task on average; twelve initially unparseable
attempts were repaired by the fixed retry rule. Gemma used one round and
208.6 output tokens per task, with no generation, execution, parsing, or
transport failures. Code execution exceeds the fixed-budget state by 8.92
points on Qwen (95\% cluster interval, 7.14--10.71) and 0.75 points on Gemma
(0.25--1.33). It is an exact-execution reference with access to the full
records and a trusted sandbox, not an equal-state or equal-compute baseline.

\FloatBarrier

\section{Statistical Analysis and Evaluation Scope}
\label{app:statistical_analysis}

\subsection{Evaluation and statistical analysis}

Task accuracy is the fraction of final choices matching the independently
derived gold answer. Context-packing, model-call, and parsing failures stay in
the prespecified denominator and receive score zero. Because tasks from the
same source window are related, all reported confidence intervals for task
accuracy and paired differences use 10,000 source-window-cluster bootstrap
resamples. The estimand is task-weighted accuracy; the 174 windows, not the
number of arm rows, are the independent clusters.

For numerical aggregation, relative error is
\begin{equation}
\operatorname{relerr}(\widehat y,y)
=
\frac{|\widehat y-y|}{\max(1,|y|)}.
\label{eq:app_relerr}
\end{equation}
The budget study reports the root mean square of this quantity. The learned
state comparison reports its median over 100 streams at each length. Those two
summaries answer different questions and are not compared across tables.

For the main paired contrasts, we additionally use a two-sided source-window
sign-flip randomization test. The statistic is the task-weighted mean paired
score difference, but the sign is randomized once per source window,
preserving all within-window dependence. We use 100,000 Monte Carlo draws, a
plus-one correction, and Holm correction over the eight contrasts in
Table~\ref{tab:app_significance}. The analysis rejects non-terminal rows,
duplicate item--arm keys, mismatched item sets, and mismatched source windows;
all input SHA-256 values are recorded with the result.

\begin{table}[ht]
\centering
\small
\setlength{\tabcolsep}{5pt}
\caption{Paired source-window randomization tests. Differences are left minus
right in percentage points. $p_{\mathrm{Holm}}$ corrects the eight displayed
tests.}
\label{tab:app_significance}
\begin{tabular}{L{0.58\textwidth}rr}
\toprule
Contrast & Difference & $p_{\mathrm{Holm}}$ \\
\midrule
Qwen budgeted $-$ exact (3,969 tasks) & $-4.03$ & $8.0\times10^{-5}$ \\
Gemma budgeted $-$ exact (3,969 tasks) & $-0.83$ & $8.0\times10^{-5}$ \\
Qwen budgeted $-$ direct context (174 tasks) & $+63.22$ & $8.0\times10^{-5}$ \\
Qwen budgeted $-$ CoT context (174 tasks) & $+60.92$ & $8.0\times10^{-5}$ \\
Gemma budgeted $-$ direct context (174 tasks) & $+56.32$ & $8.0\times10^{-5}$ \\
Gemma budgeted $-$ CoT context (174 tasks) & $+63.22$ & $8.0\times10^{-5}$ \\
Qwen code execution $-$ budgeted (1,200 tasks) & $+8.92$ & $8.0\times10^{-5}$ \\
Gemma code execution $-$ budgeted (1,200 tasks) & $+0.75$ & $0.00792$ \\
\bottomrule
\end{tabular}
\end{table}

\subsection{Scope of the evaluation}

The following boundaries apply to every result:
\begin{itemize}
    \item The evaluated aggregate-to-decision path starts from released,
    structured records with stable canonical identities and a precompiled
    operator specification. It does not establish the accuracy of extracting
    those identities from unrestricted prose.
    \item The fixed budget is per operand state or active group. Total
    aggregation memory grows with the number of simultaneously maintained
    operands and groups, even though each state does not grow with stream
    length or set cardinality.
    \item The state is updated alongside the frozen model's forward
    computation. Its readout is appended as explicit evidence, and decoding
    continues from the retained KV cache; no fused Transformer kernel or
    hidden-state modification is claimed.
    \item JMLE relation results are consumed only when the optimizer returns a
    finite feasible estimate satisfying its convergence checks. Failed
    readouts remain failures; inclusion--exclusion is not used silently.
    \item The learned-state comparison matches carried-state bytes. The
    crossover study compares one HLL state with serialized UTF-8 identity
    payload. Neither result is a total-system memory comparison.
\end{itemize}

\end{document}